\documentclass[12pt]{article}
\usepackage[T1]{fontenc}
\usepackage[utf8]{inputenc}
\usepackage[margin=1in]{geometry}
\usepackage[toc,page]{appendix}

\usepackage{amssymb}
\usepackage{tabularx}
\usepackage{tikz}
\usetikzlibrary{patterns,hobby}
\usepackage{amsmath}
\usepackage{algorithm}
\usepackage{algpseudocode}
\usepackage{mathtools}
\usepackage{enumerate}

\usepackage{url}
\usepackage{hyperref}
\hypersetup{colorlinks,allcolors=[rgb]{0,0.13672,0.95703}}
\usepackage{graphicx}
\usepackage{fullpage}
\usepackage{wrapfig}
\usepackage{color,soul}
\usepackage{float}
\usepackage{setspace}
\usepackage{subcaption}
\usepackage{tikz,pgfplots}
\pgfplotsset{compat=newest}
\usepackage{etoolbox}
\usepackage{comment}
\ifx\assumption\undefined
\newtheorem{assumption}{Assumption}
\fi

\hypersetup{
	colorlinks,
	linkcolor={red!40!gray},
	citecolor={blue!40!gray},
	urlcolor={blue!70!gray}
}

\usepackage{xcolor}
\definecolor{myorange}{RGB}{255,165,0}
\definecolor{myred}{RGB}{255,0,0}
\definecolor{myblue}{RGB}{0,0,255}
\definecolor{mycyan}{RGB}{0,255,255}
\definecolor{mybrown}{RGB}{165,42,42}
\definecolor{mygreen}{RGB}{0,128,0}

\title{
Perfect Parallelization in Mini-Batch SGD with Classical Momentum Acceleration

}

\author{
  Sachin Garg \\
  University of Michigan \\
  \texttt{sachg@umich.edu}
  \and
  Micha{\l} Derezi{\'n}ski \\
  University of Michigan \\
  \texttt{derezin@umich.edu}
}
\date{}

\newcommand{\Deltab}{\mathbf{\Delta}}

\DeclareMathOperator{\Pib}{\mathbf{\Pi}}

\newcommand{\Lambdab}{\boldsymbol{\Lambda}}

\def\Rb{{\mathbf{R}}}

\def\Sigmab{\mathbf{\Sigma}}

\def\g {\mathbf{g}}

\def\W{\mathbf W}

\def\Rb{\mathbf R}

\def\Y{\mathbf Y}

\def\T{\mathbf T}
\def\R{\mathbf R}

\def\K{\mathbf K}

\ifx\BlackBox\undefined
\newcommand{\BlackBox}{\rule{1.5ex}{1.5ex}}  
\fi
\DeclareMathOperator*{\argmin}{\mathop{\mathrm{argmin}}}

\DeclareMathOperator*{\diag}{\mathop{\mathrm{diag}}}
\def\x{\mathbf x}

\def\y{\mathbf y}

\def\b{\mathbf b}

\def\w{\mathbf w}
\def\v{\mathbf v}

\def\zero{\mathbf 0}

\def\B{\mathbf B}

\def\A{\mathbf A}
\def\C{\mathbf C}
\def\U{\mathbf U}

\def\D{\mathbf D}
\def\V{\mathbf V}
\def\Rb{\mathbf R}
\def\M{\mathbf M}

\def\Z{\mathbf Z}

\def\I{\mathbf I}

\def\A{\mathbf A}

\def\P{\mathbf P}
\def\Q{\mathbf Q}

\def\E{\mathbb E}

\def\R{\mathbb R} 
\def\tr{\mathrm{tr}}

\def\rank{\mathrm{rank}}

\let\origtop\top
\renewcommand\top{{\scriptscriptstyle{\origtop}}} 

\definecolor{silver}{cmyk}{0,0,0,0.3}
\definecolor{yellow}{cmyk}{0,0,0.9,0.0}
\definecolor{reddishyellow}{cmyk}{0,0.22,1.0,0.0}
\definecolor{black}{cmyk}{0,0,0.0,1.0}
\definecolor{darkYellow}{cmyk}{0.2,0.4,1.0,0}
\definecolor{darkSilver}{cmyk}{0,0,0,0.1}
\definecolor{grey}{cmyk}{0,0,0,0.5}
\definecolor{darkgreen}{cmyk}{0.6,0,0.8,0}

\newcommand{\Green}[1]{{\color{darkgreen}  {#1}}}
\newcommand{\Blue}[1]{\color{blue}{#1}\color{black}}
\newcommand{\Brown}[1]{{\color{brown}{#1}\color{black}}}

\ifx\proof\undefined
\newenvironment{proof}{\par\noindent{\bf Proof\ }}{\hfill\BlackBox\\[2mm]}
\fi

\ifx\theorem\undefined
\newtheorem{theorem}{Theorem}
\fi

\ifx\example\undefined
\newtheorem{example}{Example}
\fi

\ifx\property\undefined
\newtheorem{property}{Property}
\fi

\ifx\lemma\undefined
\newtheorem{lemma}{Lemma}
\fi

\ifx\proposition\undefined
\newtheorem{proposition}{Proposition}
\fi

\ifx\remark\undefined
\newtheorem{remark}{Remark}
\fi

\ifx\corollary\undefined
\newtheorem{corollary}{Corollary}
\fi

\ifx\definition\undefined
\newtheorem{definition}{Definition}
\fi

\ifx\conjecture\undefined
\newtheorem{conjecture}[theorem]{Conjecture}
\fi

\ifx\axiom\undefined
\newtheorem{axiom}[theorem]{Axiom}
\fi

\ifx\claim\undefined
\newtheorem{claim}[theorem]{Claim}
\fi

\ifx\assumption\undefined
\newtheorem{assumption}{Assumption}
\fi

\begin{document}
\maketitle
\begin{abstract}
Accelerating stochastic gradient methods with classical momentum schemes, such as Polyak's heavy ball, has proven highly successful in training large-scale machine learning models, particularly when combined with the hardware acceleration of large mini-batch computations. Yet, the effect of classical momentum on stochastic mini-batch optimization has been poorly understood theoretically, with prior works requiring strong noise assumptions and extremely large mini-batches. In this work, we develop a general theory of stochastic momentum acceleration for optimizing over quadratics in the interpolation regime, a popular abstraction for studying deep learning dynamics which also includes classical methods such as randomized Kaczmarz and coordinate descent. Our framework encompasses both heavy ball and Nesterov-style momentum, allows for arbitrary mini-batch sizes, and makes minimal assumptions on the stochastic noise. In particular, we show that acceleration from classical momentum is directly proportional to the gradient mini-batch size (up to a natural saturation point), thereby enabling perfect parallelization of mini-batch computations. Our theory also provides a simple choice for the momentum parameter, which is shown to be effective empirically.
\end{abstract}

\section{Introduction}
Consider the following optimization task which naturally arises in training machine learning models:
\begin{align}
    \min_{\w \in \R^d}\ F(\w) := \E_{\xi \sim \mathcal D}\big[f(\w;\xi)\big], \label{e:general_main}
\end{align}
where $\w$ represents the model parameters  and $\xi\sim \mathcal D$ represents the stochastic noise. Problem \eqref{e:general_main} provides a powerful abstraction for designing and analyzing algorithms for model training, particularly in the context of deep learning and large language models \cite{krizhevsky2012imagenet,he2016deep,devlin2019bert,vaswani2017attention}. 

When given oracle access to the full gradient $\g(\w) = \nabla F(\w)$ of the function $F$, problem \eqref{e:general_main} is very well understood. Here, the simplest strategy is to iteratively follow the direction of steepest descent, \begin{align}
\w_{t+1} = \w_t - \alpha\,\g(\w_t).\label{eq:GD}
\end{align}
When $F$ is smooth and strongly convex,  this Gradient Descent algorithm (GD) attains linear convergence, i.e., it reaches an $\epsilon$-approximate solution after $O(\kappa_F\log(1/\epsilon))$ iterations, where $\kappa_F$ is an appropriate condition number of the deterministic function $F$.  In 1964, Polyak \cite{polyak1964some} showed that the convergence rate of GD can be improved by introducing momentum into its descent trajectory:
\begin{align}
    \w_{t+1} = \w_t - \alpha\,\g(\w_t) + \beta\,(\w_t-\w_{t-1}).
    \label{eq:HBM}
\end{align}
With an appropriate choice of step size $\alpha$ and momentum parameter $\beta$, this Heavy Ball Momentum algorithm (HBM) reaches an $\epsilon$-approximate solution after $O(\sqrt{\kappa_F}\log(1/\epsilon))$ iterations. In 1983, Nesterov \cite{nesterov1983method} proposed a refinement of this scheme, known as Nesterov Accelerated Gradient (NAG), which essentially incorporates one additional gradient step to both $\w_t$ and $\w_{t-1}$ in the momentum.

For large-scale problems, instead of the full gradient we often only have access to stochastic gradients $\nabla f(\w,\xi)$. By aggregating $m$ such gradients into the vector $\g(\w) = \frac1m\sum_{i=1}^m\nabla f(\w;\xi_i)$, the GD update \eqref{eq:GD} turns into the popular mini-batch Stochastic Gradient Descent algorithm (SGD, \cite{robbins1951stochastic}), which has become a workhorse of modern machine learning. The success of SGD in deep learning has been partly attributed to the fact that it is particularly effective in the \emph{interpolation regime} \cite{belkin2018overfitting,ma2018power,attia2025fast}, i.e., when all of the functions $f(\cdot,\xi)$ can be simultaneously minimized (the model fits perfectly to the training data). In this setting, for smooth and strongly convex objectives, SGD attains linear convergence with $O(\kappa_f\log(1/\epsilon))$ iteration complexity, where $\kappa_f$ is a corresponding condition number of the stochastic function $f$. The use of large mini-batch size $m$ allows us to leverage parallel computing architectures while reducing the variance of the stochastic gradient \cite{li2014efficient}, however its effect on the iteration complexity of vanilla SGD is limited \cite{byrd2012sample} and not fully understood in practice \cite{golmant2018computational}.

A natural strategy for improving the performance of mini-batch SGD is to incorporate the classical momentum update \eqref{eq:HBM} of Polyak, or its Nesterov counterpart. This approach has proven highly successful empirically, particularly when paired with large mini-batches \cite{sutskever2013importance,he2016deep,sandler2018mobilenetv2,fu2023and}, and it inspired state-of-the-art training algorithms such as Adam \cite{kingma2014adam} as well as many other more elaborate acceleration schemes for SGD \cite{gupta2018shampoo,vyas2024soap,jordan2024muon}. However, this has led to a fundamental theory-practice gap: While there is significant theoretical progress on accelerating SGD through carefully designed schemes \cite{derezinski2025randomized, han2025pseudoinverse, alderman2024randomized,allen2018katyusha, jin2018accelerated, jain2017accelerating, van2017fastest, liu2016accelerated, bubeck2015geometric}, the SGD algorithms used in practice still largely rely on the simple and classical momentum based on HBM, whose remarkable effectiveness remains poorly understood. Initial efforts in that direction were done by Loizou and Richtarik \cite{loizou2020momentum}, who showed linear convergence of SGD with HBM, but without attaining any acceleration. Bollapragada, Chen and Ward \cite{bollapragada2022fast} showed that HBM-accelerated mini-batch SGD can exhibit faster convergence on quadratic tasks, with related and follow-up works by \cite{lee2022trajectory,pan2023accelerated,tang2023acceleration,zeng2024adaptive}. However, in order to attain this acceleration, all of these prior works require either strong additional assumptions on the stochastic noise or prohibitively large mini-batch sizes that depend heavily on problem~parameters.

\subsection{Our contributions}

In this work, we address the above theory-practice gap by developing a general framework for analyzing the convergence of momentum-based stochastic processes, which encompasses mini-batch SGD with both HBM and NAG acceleration for stochastic quadratic minimization in the interpolation regime. Our analysis leads to a strikingly simple conclusion about the role of classical momentum in large-scale stochastic optimization:
\begin{center}
    \emph{
    Classical momentum enables perfect parallelization for mini-batch SGD\\ in the interpolation regime.
    }
\end{center}
\textbf{Main result.} Concretely, our main technical result (Theorem \ref{t:L2_iterate_0}) shows that properly tuned SGD using mini-batch size $m$ and either HBM or NAG momentum attains $O\big((\kappa_f/m + \sqrt{\kappa_f})\log(\kappa_f/\epsilon)\big)$ iteration complexity in the interpolation regime. Note that for $m=1$ (no mini-batching) this result essentially recovers the previously mentioned linear convergence for SGD, but as we increase $m$, the iteration complexity improves proportionally to $1/m$ until it saturates at an accelerated $\sqrt{\kappa_f}$ rate. Thus, up to the point of saturation, the overall work (number of stochastic gradients computed) remains constant, allowing us to fully exploit the embarrassingly parallel mini-batch computations. Such perfect parallelization does not in general hold for mini-batch SGD without acceleration, even over quadratics in the interpolation regime. Thus, our analysis provides an explanation for the success of momentum acceleration in allowing deep learning optimizers such as Adam to benefit from large mini-batches and highly parallelizable GPU~architectures.

\textbf{Applications.} Our analysis framework can also be directly applied to accelerate several popular SGD algorithms for solving large-scale regression tasks, such as randomized Kaczmarz \cite{strohmer2009randomized} and coordinate descent \cite{gower2015randomized}. In particular, we use our theory to develop a simple momentum-accelerated block coordinate descent algorithm, and 
we evaluate it on several kernel ridge regression tasks. Our numerical results confirm the perfect parallelization phenomenon, and they also show that a simple choice of $\beta = 1 - 1/m$, as motivated by our theory, already leads to meaningful acceleration. We also compare classical momentum with an optimally tuned state-of-the-art acceleration scheme for block coordinate descent \cite{derezinski2025randomized}, showing that, surprisingly, in certain large mini-batch regimes classical momentum yields substantially better~performance.

\subsection{Overview of our techniques} 

\paragraph{Limitations of prior approaches.} 
The convergence behavior of a momentum-accelerated stochastic method follows a dynamical system that is characterized by a product of non-symmetric random matrices. Existing matrix concentration results can be used to bound the spectral norm of a product of random matrices in terms of the product of the spectral norms of the expected values of the matrices \cite{huang2022matrix}. Unfortunately, this strategy proves vacuous for momentum acceleration, as the product of norms diverges to infinity even when the norm of the product converges to zero, because of the lack of symmetry. 
To address this, \cite{bollapragada2022fast} diagonalize the matrices so that they can use matrix product concentration effectively.  However, passing from the original matrices to the diagonalized form incurs a penalty factor that appears both in the convergence bound and in the required mini-batch size, and can blow up independently of the condition number of the problem. In fact, the penalty can even be infinite, as the underlying matrices may not be diagonalizable. To get around this, \cite{pan2023accelerated} diagonalize the matrices to a Jordan form, however this still does not avoid a potentially large penalty factor in both the convergence  bound and the mini-batch size.
\paragraph{Our approach.} In order to get around these issues, we abandon both the diagonalization step and matrix product concentration. Instead, we use the Schur decomposition together with a matrix recursion to bound the original matrix product directly. Our analysis consists of two key components, following a bias-variance decomposition of the expected error.
\vspace{0.3cm}

~\quad\textbf{1.~Bias: \textit{Convergence of expected iterates via Schur decomposition}.} For the expected iterates, it suffices to analyze the product of certain \emph{deterministic} non-symmetric matrices. Here, we start by rearranging the matrices into a block-diagonal form, and carefully analyze both their real and complex eigenvalues in order to bound the spectral radius. Next, we turn this into a spectral norm bound for the matrix product. In this step, instead of using diagonalization as done in prior work, we rely on the Schur decomposition \cite{plemmons1988matrix}, thereby avoiding condition number factors in our bound.
\vspace{0.3cm}

~\quad\textbf{2.~Variance: \textit{Convergence of expected error via matrix recursion}.} In order to control the random iterates, we must bound the stochastic error incurred at each iteration of the process. Instead of relying on black-box matrix concentration results, we do this by deriving a recursion for the expectation of the product of the random matrices, which relates its spectral norm to all monomials of the individual block matrices from the block-diagonal representations derived in the bias analysis. Unfolding this recursion proves quite technical as it requires separately handling real and complex eigenvalues of the diagonal blocks.

Our analysis is not specific to one algorithm, but rather provides a general framework for effectively characterizing the convergence behavior of stochastic momentum methods. We illustrate this by stating our results for a class of stochastic processes which includes both HBM and NAG as special cases, and can be instantiated to provide convergence guarantees for a variety of SGD-type algorithms.

\section{Further related work}
\label{s:related_work}
In addition to the classical HBM and NAG momentum, a variety of more sophisticated momentum-based methods have been proposed in the literature \cite{derezinski2025randomized, han2025pseudoinverse, alderman2024randomized, allen2018katyusha, jin2018accelerated, jain2017accelerating, van2017fastest, liu2016accelerated, bubeck2015geometric}. Concurrently, substantial theoretical progress has been made toward understanding the convergence properties of stochastic gradient descent (SGD) equipped with various forms of momentum.

Alongside these developments, a consistent empirical finding is the accelerated convergence achieved by incorporating classical momentum schemes such as HBM or NAG into mini-batch SGD \cite{sutskever2013importance, he2016deep, sandler2018mobilenetv2, fu2023and}. Notwithstanding this strong empirical evidence, the theoretical understanding of the mechanisms underlying such acceleration remains incomplete. In particular, even in the setting of quadratic objectives, a comprehensive explanation of when and why these classical momentum methods yield acceleration is still lacking.

A substantial body of work has analyzed the behavior of mini-batch SGD with HBM and NAG \cite{gupta2024nesterov,zeng2024adaptive,pan2023accelerated,tang2023acceleration,bollapragada2022fast,loizou2020momentum,can2019accelerated,yan2018unified,gadat2018stochastic,kidambi2018insufficiency,flammarion2015averaging}. Many of these studies establish that, in general, SGD with $O(1)$ size mini-batches and with such momentum does not admit provable improvements over its non-momentum counterpart \cite{kidambi2018insufficiency,loizou2017linearly,jain2017accelerating,wang2023marginal}, attributing observed empirical gains primarily to variance reduction induced by large mini-batches. Despite these negative results, more recent works have revisited this question under restrictive settings, including large mini-batches, carefully tuned step sizes and momentum parameters, and specific assumptions on the data and noise \cite{bollapragada2022fast,tang2023acceleration,pan2023accelerated,zeng2024adaptive,lee2022trajectory}. For instance, \cite{zeng2024adaptive} demonstrate accelerated convergence of randomized linear solvers with HBM via adaptive parameter tuning; however, their analysis relies on distributional assumptions on the sketching matrices that may not hold for commonly used subsampling schemes, and their guarantees do not explicitly quantify the improvement in convergence rates. Similarly, \cite{lee2022trajectory} establish acceleration for mini-batch SGD only in an asymptotic regime and under additional assumptions such as random, orthogonally invariant data matrices.

The works most closely related to ours are \cite{bollapragada2022fast,pan2023accelerated,tang2023acceleration}. In particular, \cite{bollapragada2022fast} establish accelerated convergence of mini-batch SGD with HBM under the requirement that the batch size satisfies $m = \tilde\Omega(\kappa_f \sqrt{\kappa_F})$, along with additional stable diagonalization assumptions. In a different setting, \cite{tang2023acceleration} prove acceleration for inconsistent linear systems, but require a significantly larger batch size, specifically $m = \tilde{\Omega}(\kappa_f^{3})$. More recently, \cite{pan2023accelerated} show that mini-batch SGD with HBM has iteration complexity $\tilde O\big(\sqrt{\kappa} + \frac{d}{m\epsilon} \big)$ to reach $\epsilon$-accuracy, showing sublinear convergence that improves with mini-batch size. They do not consider the interpolation regime, but rather, rely on  other assumptions such as anisotropic gradient noise and carefully designed step-size schedules. Since we are focused on the interpolation regime and attaining linear convergence, their results are largely incomparable~to~ours. 
\section{Stochastic momentum acceleration framework}\label{s:framework}
\paragraph{Notation.}
We represent vectors by lowercase boldface letters and matrices by uppercase boldface letters. For symmetric matrices, we use Loewner ordering and say $\A\preceq \B$ if $\B-\A$ is a positive semidefinite (psd) matrix. $\|\A\|$ denotes spectral norm, $\|\A\|_F$ denotes Frobenius norm of matrix $\A$. For a psd  matrix $\A$, $\|\x\|_{\A} = \sqrt{\x^\top\A\x}$ denotes the Mahalanobis norm of $\x$. We let $\A^\dagger$ denote the Moore-Penrose pseudoinverse of $\A$. We use $\tilde O$, $\tilde\Omega$ notation to hide logarithmic factors.

\paragraph{Stochastic contraction process.} We start by defining a class of stochastic dynamical systems called stochastic contraction processes \cite{derezinski2026last}. This framework can be instantiated to represent the dynamics of many stochastic algorithms, as described below, and thus serves as a useful abstraction for our theory.
\begin{definition}\label{d:stochastic-contraction}
    For a random psd matrix $\Pib$ such that $\zero \preceq\Pib\preceq \I$ and scalar $\alpha\in[0,1]$, we let $\mathcal{A}(\Deltab;\alpha) := (\I-\alpha\Pib)\Deltab$ define a stochastic contraction with average rate $\bar\Pib=\E[\Pib]$. We call $\Pib$ the stochastic rate matrix of $\mathcal{A}$. Given a sequence of independent $\mathcal{A}_t$ with the same $\bar\Pib$ and $\alpha$, we call $(\{\mathcal{A}_t\}_{t\geq 0},\alpha)$ a stochastic contraction process.
\end{definition}

To see how this definition corresponds to an SGD algorithm, consider the minimization problem \eqref{e:general_main} over $F(\w) = \E_{\xi\sim\mathcal D}[f(\w;\xi)]$, where for every $\xi$, function $f(\w;\xi)$ is a convex quadratic with respect to $\w$ that is minimized at some $\w^*\in\R^d$. As long as function $f(\cdot;\xi)$ is $L$-smooth, i.e., $\|\nabla f(\w,\xi) - \nabla f(\v,\xi)\|\leq L\|\w-\v\|$, then the evolution of the error vector $\Deltab_t = \w_t-\w^*$ resulting from the SGD update \eqref{eq:GD} with $\g(\w_t) = \frac1{L}\nabla f(\w_t;\xi_t)$ follows a stochastic contraction process:
\begin{align}
    \Deltab_{t+1} = \mathcal{A}_t(\Deltab_t;\alpha),\qquad \text{where}\qquad \mathcal{A}_t(\Deltab;\alpha) = \Big(\I - \frac{\alpha}{L}\nabla^2 f(\w_t;\xi_t)\Big)\Deltab.\label{eq:sgd-dynamics}
\end{align}
The stochastic rate matrix defining the stochastic contraction process $\mathcal{A}_t$ is $\Pib_t = \frac1L\nabla^2 f(\w_t;\xi_t)$, and its average rate is $\bar\Pib =\frac1L\nabla^2F(\w^*)$. Note that the Hessian is the same everywhere since both $f(\cdot;\xi_t)$ and $F$ are assumed to be quadratic.

\paragraph{Stochastic momentum process.}
We next show how to introduce momentum into this process by incorporating information from two preceding steps. We define these dynamics in sufficient generality to capture both HBM and NAG momentum as special cases.

\begin{definition}\label{d:stochastic-momentum}
    Given a stochastic contraction process
    $(\{\mathcal{A}_t\}_{t\geq 0},\alpha)$, initial vectors $\Deltab_{-1},\Deltab_0$, and parameters $\beta,\omega\geq 0$, define the corresponding stochastic momentum process recursively as:
\begin{align}
    \Deltab_{t+1} = \mathcal{A}_t(\Deltab_t;\alpha) + \beta\Big[\mathcal{A}_t(\Deltab_t;\omega) - \mathcal{A}_{t-1}(\Deltab_{t-1};\omega)\Big].\label{eq:stochastic-momentum}
\end{align}
\end{definition}
Note that if we set the momentum parameter $\beta$ to zero, then we recover the usual SGD dynamics of a stochastic contraction process. On the other hand, if we set $\beta>0$ but $\omega = 0$, then the momentum term simplifies to $\beta[\Deltab_t - \Deltab_{t-1}]$, which precisely corresponds to Polyak's heavy-ball momentum (HBM). On the other hand, by choosing $\omega = \alpha$ we introduce one more SGD step to the momentum term, thereby recovering Nesterov's momentum (NAG). For example, we can use this framework to model HBM acceleration for the SGD process $\{\mathcal{A}_t\}_{t\geq 0}$ from \eqref{eq:sgd-dynamics} by letting 
\begin{align*}
    \Deltab_t = \w_t-\w^*,\quad\text{ and }\quad \w_{t+1} = \w_t - \frac{\alpha}{L}\nabla f(\w_t;\xi_t) + \beta(\w_t - \w_{t-1}).
\end{align*}
It is easy to verify that such a sequence $\{\Deltab_t\}_{t\geq 0}$ is a stochastic momentum process associated with $(\{\mathcal{A}_t\}_{t\geq 0},\alpha)$, momentum parameter $\beta$, and $\omega = 0$. A similar derivation follows immediately  for Nesterov's momentum by setting $\omega = \alpha$.

\paragraph{Example 1: Randomized Kaczmarz.}
As a concrete example of an algorithm that falls directly under our framework, we consider the randomized Kaczmarz (RK) algorithm \cite{strohmer2009randomized}, along with its block versions \cite{needell2014paved}. Let $\A\in\R^{n\times d}$ and $\y\in\R^n$ define a linear regression problem with the loss function $F(\w) = \|\A\w-\y\|^2$. Also, let $\mathcal D$ be a distribution over subsets $S\subseteq \{1,...,n\}$, and consider the following update:
\begin{align}
    \w_{t+1}\ = \argmin_{\w:\, \A_{S_t}\w=\y_{S_t}} \|\w-\w_t\|^2\ =\ \w_t - \A_{S_t}^\dagger(\A_{S_t}\w_t-\y_{S_t}),\qquad S_t\sim\mathcal D,\label{eq:rk}
\end{align}
where $\A_{S_t}$ denotes the submatrix consisting of rows indexed by $S_t$. In particular, when $S\sim \mathcal D$ generates single-sample sets such that $\Pr[S\!=\!\{i\}] \propto \|\A_{i,:}\|^2$, then this recovers classical randomized Kaczmarz. When there exists $\w^*$ such that $\A\w^*=\y$ (i.e., in the interpolation regime), the update \eqref{eq:rk} can be cast as a stochastic contraction process $(\{\mathcal A_t\}_{t\geq 0},1)$ by letting $\Pib_t = \A_{S_t}^\dagger\A_{S_t}$ and $\Deltab_t = \w_t-\w^*$. In particular, for randomized Kaczmarz, a simple calculation shows that this process has an average rate $\bar\Pib = \E[\Pib_t]=\A^\top\A/\|\A\|_F^2$. Similar or better rates have been derived for block versions of Kaczmarz \cite{dy24,derezinski2025randomized}, and other extensions such as sketch-and-project \cite{derezinski2024sharp,rathore2024have,rathore2026turbocharging}. By applying momentum via Definition \ref{d:stochastic-momentum} to $(\{\mathcal A_t\}_{t\geq 0},1)$, we can precisely recover convergence dynamics of the error vector of the accelerated Kaczmarz algorithm considered by \cite{bollapragada2022fast}.

\paragraph{Example 2:  Coordinate Descent.} Another application of our framework comes in the context of solving positive semidefinite linear systems, such as in kernel ridge regression. For a positive definite matrix $\K\in\R^{n\times n}$ and $\y\in\R^n$, consider the (block) coordinate descent (CD, \cite{leventhal2010randomized}) update:
\begin{align}
    \w_{t+1} = \w_t - \I_{S_t}^\top(\K_{S_t,S_t})^{\dagger}(\K\w_t-\y)_{S_t},\qquad S_t\sim\mathcal D,\label{eq:cd}
\end{align}
which only updates the coordinates of $\w_t$ indexed by $S_t$. Letting $\w^*$ denote the solution of the linear system $\K\w=\y$, we can cast \eqref{eq:cd} as a stochastic contraction process $(\{\mathcal{A}_t\}_{t\geq 0},1)$ by choosing $\Deltab_t = \K^{1/2}(\w_t-\w^*)$ and $\Pib_t = \K_{S_t}^{1/2}(\K_{S_t,S_t})^{\dagger}\K_{S_t}^{1/2}$. Here, the canonical distribution $\mathcal D$ for single-sample coordinate descent is $\Pr[S\!=\!\{i\}] \propto K_{i,i}$, which yields average rate $\bar\Pib = \K/\tr(\K)$.

\paragraph{Convergence rate of a stochastic contraction process.} The above examples are of base algorithms, i.e., without momentum acceleration. Before we analyze their accelerated variants through the stochastic momentum framework, we establish their base convergence rate via our general formalism.
\begin{proposition}\label{p:simple-rate}
    Given a stochastic contraction process $(\{\mathcal A_t\}_{t\geq 0},\alpha)$ with average rate $\bar\Pib$, and an initial vector $\Deltab_0\in\mathrm{range}(\bar\Pib)$, consider the sequence $\Deltab_{t+1}=\mathcal A_t(\Deltab_t;\alpha)$. Then, we have:
    \begin{align*}
        \E\,\|\Deltab_t\|^2 \leq \Big(1 -
        \frac{1}{\kappa}\Big)^t\|\Deltab_0\|^2,\qquad\text{where}\quad\kappa = \alpha/\lambda_{\min}^{+}(\bar\Pib).
    \end{align*}
\end{proposition}
We will refer to $\kappa$ as the stochastic condition number associated with the process $(\{\mathcal{A}_t\}_{t\ge 0},\alpha)$. For instance, by plugging in our concrete examples of (single-sample) randomized Kaczmarz and coordinate descent \cite{strohmer2009randomized}, we recover the classical convergence guarantees for these methods, with stochastic condition numbers $\kappa_{\mathrm{RK}} = \|\A\|_F^2\|\A^\dagger\|^2$ and $\kappa_{\mathrm{CD}} = \tr(\K)\|\K^{\dagger}\|$.

\paragraph{Effect of mini-batching.} The stochastic contraction framework is naturally amenable to mini-batching of the updates via simple averaging of the contractions. It is thus natural to ask how mini-batching affects the convergence dynamics. 

\begin{definition}\label{d:mini-batching}
    Given a stochastic contraction $\mathcal {A}$ with average rate $\bar\Pib$, let $\mathcal{A}^{[m]}$ denote the average of $m$ independent copies of $\mathcal{A}$, i.e.,
    \begin{align}
        \mathcal{A}^{[m]}(\Deltab;\alpha) := \frac1m\sum_{i=1}^m\mathcal{A}^{(i)}(\Deltab;\alpha),\qquad\text{where}\quad \mathcal{A}^{(i)}\sim \mathcal{A}.\label{eq:mini-batching}
    \end{align}
\end{definition}
For example, if $(\{\mathcal{A}_t\}_{t\geq 0},\alpha)$ is a stochastic contraction process for the SGD algorithm given~by~\eqref{eq:sgd-dynamics}, then $(\{\mathcal{A}_t^{[m]}\}_{t\geq 0},\alpha)$ corresponds to mini-batch SGD where $\g(\w_t) = \frac1{mL}\sum_{i=1}^m\nabla f(\w_t,\xi_{t}^{(i)})$. Note that mini-batching is different from choosing a larger block set $S_t$ in randomized Kaczmarz or coordinate descent, since those schemes do more than just averaging the samples (such as inverting the matrices $\A_{S_t}$ and $\K_{S_t,S_t}$). In fact, combining blocking and mini-batching in RK and CD is a very natural strategy for making these methods scalable to parallel architectures, as discussed later.
\begin{proposition}\label{p:mini-batching}
    If $\mathcal {A}$ is a stochastic contraction  with average rate $\bar\Pib$, then $\mathcal{A}^{[m]}$ is also a stochastic contraction with the same average rate $\bar\Pib$. Moreover, if $\Pib^{[m]}$ is the stochastic rate matrix associated with $\mathcal{A}^{[m]}$, then it follows that:
    \begin{align}
        \E\Big[\big(\Pib^{[m]}-\bar\Pib\big)^2\Big] \preceq \frac{1}{m}\cdot\bar\Pib(\I-\bar\Pib).\label{eq:mini-batching-variance}
    \end{align}
\end{proposition}
Proposition \ref{p:mini-batching}, together with Proposition \ref{p:simple-rate}, illustrates why mini-batching of SGD does not in general lead to faster convergence rates, since the average rate matrix $\bar\Pib$ (and thus also the associated stochastic condition number) remains the same regardless of $m$. However, mini-batching does reduce the variance of the stochastic noise matrix, as shown in \eqref{eq:mini-batching-variance}, which will be crucial for the stochastic momentum dynamics. We defer the proofs of Propositions \ref{p:simple-rate} and \ref{p:mini-batching} to Appendix \ref{s:prop_proofs}.

\section{Main result}\label{s:main_result}
We are now ready to state our main result, which characterizes the convergence dynamics of a stochastic momentum process (Definition \ref{d:stochastic-momentum}) corresponding to stochastic contractions with mini-batching (Definitions \ref{d:stochastic-contraction} and \ref{d:mini-batching}). Crucially, our result makes no assumptions on the mini-batch size $m$ or the stochastic noise, and applies to both heavy-ball and Nesterov momentum.
\begin{theorem}\label{t:L2_iterate_0}
    Consider a stochastic contraction process 
    $(\{\mathcal{A}_t\}_{t\ge0},\alpha)$ with $\alpha\in[0,1]$, average rate $\bar\Pib$, and condition number $\kappa = \alpha/\lambda_{\min}^+(\bar\Pib)$. For $m\geq 1$, $\Deltab_0\in\mathrm{range}(\bar\Pib)$, let $\Deltab_{-1}=\zero$, $\{\Deltab_t\}_{t\ge0}$ be a stochastic momentum process for $(\{\mathcal{A}_t^{[m]}\}_{t\ge0},\alpha)$ with $\omega\in[0,\alpha]$ and momentum parameter
    \begin{align*}
        \beta = 
        \begin{cases}
            0 & \text{if }\quad m\leq c_1,\\
            1 - \frac{c_1}{2m} & \text{if }\quad m \in [c_1, c_2\sqrt\kappa],\\
            1 - \frac{c_1}{2c_2\sqrt\kappa} & \text{if }\quad m\geq c_2\sqrt\kappa,
        \end{cases}
    \end{align*}
    where $c_1,c_2>0$ are absolute constants. Then, there are absolute constants $C,c>0$ such that:
    \begin{align*}
        \E\,\|\Deltab_t\|^2 \leq Ct^3 \bigg(1 - c \min\Big\{\frac{ m}{\kappa},\frac1{\sqrt\kappa}\Big\}\bigg)^t\|\Deltab_0\|^2\qquad\text{for any }t\geq 1.
    \end{align*}
\end{theorem}
\begin{remark}
    In particular, we get $\E\,\|\Deltab_t\|^2 \leq \epsilon\|\Deltab_0\|^2$ after $t = O((\kappa/m + \sqrt\kappa)\log(\kappa/\epsilon))$ iterations, compared to $t = O(\kappa\log(1/\epsilon))$ without momentum. Thus, for $m\in[1,\Theta(\sqrt\kappa)]$, momentum enables perfect parallelization up to a logarithmic factor (which is caused by the $t^3$ term in the bound).
\end{remark}
\begin{remark}
    In the regime $m\in[1,\Theta(\sqrt\kappa)]$, our result yields a simple problem-independent momentum parameter $\beta = 1 - \Theta(1/m)$. Numerical experiments confirm that this choice works well in practice.
\end{remark}

\subsection{Proof sketch}\label{s:technical}
First, without loss of generality we can assume that $\alpha=1$ (otherwise, we can simply rescale $\bar\Pib$ appropriately). Also, note that the regime of $1\leq m\leq c_1$ follows immediately from Proposition \ref{p:simple-rate}. Next, let $\{\Pib_t^{[m]}\}_{t\geq 0}$ be the stochastic noise matrices associated with $\{\mathcal{A}_t^{[m]}\}_{t\ge 0}$. From Proposition~\ref{p:mini-batching}, these matrices satisfy the variance bound \eqref{eq:mini-batching-variance}. In fact, this is the only property we need out of the mini-batching, so to simplify the notation, we will drop the superscript and consider a sequence $\{\Pib_t\}_{t\geq 0}$ of matrices such that $\E\,(\Pib_t-\bar\Pib)^2\preceq\frac1m\bar\Pib(\I-\bar\Pib)$. With these conventions, after simple calculations the stochastic momentum process \eqref{eq:stochastic-momentum} can be expressed as the following recursion:
\begin{align}
    \Deltab_{t+1} = \big((1+\beta)\I-(1+\beta\omega)\Pib_t\big)\Deltab_{t} -\beta\big(\I-\omega\Pib_{t-1}\big)\Deltab_{t-1}. \label{e:sto_proc}
\end{align}
For notational inconvenience, we restrict ourselves to NAG acceleration in this proof sketch (i.e. $\omega=\alpha=1$). The detailed proof in the appendix covers HBM and NAG momentum simultaneously.
Let $\bar\Pib = \V\Lambdab\V^\top$ be the eigendecomposition of $\bar\Pib$, where $\Lambdab = \diag(\lambda_1,\lambda_2,\cdots,\lambda_d)$ and $\V\V^\top=\V^\top\V=\I$. Let $\rank(\bar\Pib) =r \leq d$, so that $\lambda_{i}=0 $ for $i>r$. Following prior works \cite{jain2017accelerating,bollapragada2022fast,pan2023accelerated}, we study the dynamics of stochastic process $\Deltab_t$ in (\ref{e:sto_proc}) via the following two-step matrix transition rule:
\begin{align}
        \begin{bmatrix}
            \V^\top\Deltab_{t} \\ \V^\top\Deltab_{t-1}
        \end{bmatrix} = \underbrace{\begin{bmatrix} (1+\beta)(\I-\V^\top\Pib_t\V) & -\beta(\I-\V^\top\Pib_{t-1}\V)\\
        \I & \zero\end{bmatrix}}_{:=\Y_t}\cdot \begin{bmatrix}
            \V^\top\Deltab_{{t-1}} \\ \V^\top\Deltab_{t-2}
        \end{bmatrix} \label{e:stoc_proc_two_step}
    \end{align}

We break down the proof of Theorem \ref{t:L2_iterate_0} into two main components, starting with the convergence of the expected iterates (bias), and then analyzing the effect of the stochastic noise (variance).

\subsubsection{Bias: Convergence of $\|\E\Deltab_t\|$} 

The convergence behavior of the expected iterates is characterized by a simpler transition rule:
\begin{align} 
     \E\begin{bmatrix}
            \V^\top\Deltab_{t} \\ \V^\top\Deltab_{t-1}
        \end{bmatrix} = \T^{t} \begin{bmatrix}
            \V^\top\Deltab_{0} \\ \V^\top\Deltab_{-1}
        \end{bmatrix},
\quad\text{where }\ \T:=\E[\Y_t] = \begin{bmatrix} (1+\beta)(\I-\Lambdab) & -\beta(\I-\Lambdab)\\
        \I & \zero\end{bmatrix}.\label{e:exp_yt}
\end{align}
 
The first thing that comes to mind is to use $\big\|\T^t\| \leq \|\T\|^t$, but, unfortunately, for non-zero $\beta$ the spectral norm of $\T$ is greater than $1$, making this naive idea ineffective. However, the spectral radius of $\T$ is less than $1$ and this paves the way ahead, suggesting a two step approach:
\vspace{0.3cm}

~\quad\textbf{Step 1:} Bound the spectral radius of $\T^{t}$, denoted as $\rho(\T^t)$;
\vspace{0.3cm}

~\quad\textbf{Step 2:} Using structural properties of $\T$, bound $\big\|\T^t\|$ in terms of $\rho(\T^t)$.

\paragraph{Step 1.} 
By carefully permuting rows and columns of $\T$, we can transform it into a block-diagonal matrix with $2$ by $2$ blocks on the diagonal. In particular, the following lemma holds:

\begin{lemma}[Block-diagonalization]\label{l:perm_0}
        There exists a permutation matrix $\P$ such that,
        \begin{align*}
           \P\T\P^\top = \diag(\T_1,\T_2,\cdots,\T_d),\quad\text{where }\ 
        \T_i :=  \begin{bmatrix}
            (1+\beta)(1-\lambda_i) & -\beta(1-\lambda_i) \\
            1&0
        \end{bmatrix}.
        \end{align*}
    \end{lemma}
Due to Lemma \ref{l:perm_0}, taking maximum over the spectral radius of each $\T_i$ will upper bound $\rho(\T)$. Interestingly, the eigenvalues of $\T_i$ depend on interaction between $\lambda_i$ and $\beta$, and can be either real or complex depending upon whether $\lambda_i \leq \frac{(1-\beta)^2}{(1+\beta)^2}$. Let $\gamma_i$ denote the eigenvalue of $\T_i$ with the larger magnitude, so that $|\gamma_i|$ equals the spectral radius of $\T_i$. We prove the following result:
\begin{lemma}[Spectral radius of $\T_i$]\label{l:spec_radius_T_0}
Let $\beta = 1-\frac{1}{\phi}$ with $\phi \geq 2$. Then,
\begin{align*}
|\gamma_i|^2 \le 
\begin{cases}
\big(1-\tfrac{1}{\phi}\big)(1-\lambda_i), 
&\text{if }\ \lambda_i \ge \tfrac{(1-\beta)^2}{(1+\beta)^2}\; \text{ (i.e., when $\gamma_i$ is complex)},\\[4pt]
\big(1-\tfrac{\phi\lambda_i}{2}\big)^2(1-\lambda_i)^2, 
& \text{otherwise}\; \text{ (i.e., when $\gamma_i$ is real).}
\end{cases}
\end{align*}
\end{lemma}
 
 It is easy to see that the maximum is attained at $i=r$, thus yielding an upper bound on $\rho(\T)$.
 
 \paragraph{Step 2.} To go from $\rho(\T^t)$ to $\big\|\T^t\|$, previous work \cite{bollapragada2022fast} assumes that $\T$ is diagonalizable, i.e., there exists an invertible matrix $\C$ such that $\C\T\C^{-1}$ is a diagonal matrix. This assumption implies,
 \begin{align*}
     \big\|\T^t\| \leq \big\|\C\|\cdot\big\|\C^{-1}\|\cdot\rho(\T^t).
 \end{align*}
Unfortunately, $\T$ need not be diagonalizable. In fact, if for any $i$, $\lambda_i = \frac{(1-\beta)^2}{(1+\beta)^2}$, then $\T$ is not diagonalizable. More problematically, it can be shown that
\begin{align*}
    \|\C\|\cdot\|\C^{-1}\| = \Omega \Big(\frac{\phi}{\delta}\Big), \quad  \text{where \ $\delta = \min_{i}\Big\{\big|\lambda_i - \frac{(1-\beta)^2}{(1+\beta)^2}\big|\Big\}$}.
\end{align*}
 
To avoid this blow up, we abandon the diagonalization approach and rely on a fundamental matrix decomposition result in linear algebra, known as Schur's decomposition.
       \begin{lemma}[Schur's decomposition \cite{plemmons1988matrix}]\label{l:schur_0}
    For every $i$, there exists a unitary matrix $\U_i$ such that
    \begin{align*}
        \U^{H}_i\T_i\U_i = \begin{bmatrix}
            \gamma_{i1} &x_i\\
            0 &\gamma_{i2}
        \end{bmatrix},
    \end{align*}
    where $\gamma_{i1}$ and $\gamma_{i2}$ are the two eigenvalues of $\T_i$, $1 \leq |x_i| \leq 3$ and $\U_i^{H}$ denotes the hermitian of $\U_i$.
\end{lemma}
Using Schur's decomposition it readily follows that for any $i$,
   \begin{align*}
       \|\T_i^t\| = \|\U_i^H\T_i^t\U_i\| \leq O(t)\cdot |\gamma_{i}|^{t-1}.
   \end{align*}
Using the above in conjunction with Lemma \ref{l:spec_radius_T_0} we show that:
  \begin{theorem}[Convergence of $\|\E\Deltab_t\|$]\label{t:exp_iterate_0} For $\rho = \big(1-\frac{\phi}{2\kappa}\big)(1-\frac{1}{\kappa})$, the transition rule (\ref{e:stoc_proc_two_step}) satisfies,
     \begin{align*}
         \|\E\Deltab_{t}\|^2 \leq O(t^2)\cdot\rho^{t-2}\|\Deltab_0\|^2.
     \end{align*} 
\end{theorem}

\subsubsection{Variance: Convergence of $\E\|\Deltab_t\|^2$}
Note that (\ref{e:stoc_proc_two_step}) implies,
\begin{align*}
    \P\begin{bmatrix}
        \V^\top\Deltab_t \\
        \V^\top\Deltab_{t-1}
    \end{bmatrix} = \Big(\underbrace{\prod_{j=0}^{t-1}\P\Y_{t-j}\P^\top}_{\Sigmab_t}\Big)\cdot\P\begin{bmatrix}
        \V^\top\Deltab_0 \\
        \V^\top\Deltab_{-1}
    \end{bmatrix}.
\end{align*}
To establish convergence of $\E\|\Deltab_t\|^2$, it suffices to bound $\|\E[\Sigmab_t^\top \Sigmab_t]\|$. Since $\Sigmab_t$ is a product of non-symmetric random matrices, upper bounding $\big\|\E[\Sigmab_t^\top\Sigmab_t]\big\|$ is challenging. Prior work \cite{bollapragada2022fast}  uses matrix product concentration \cite{huang2022matrix} and require careful tuning of the momentum parameter to keep $\T$ diagonalizable while controlling $\|\C\|\cdot\|\C^{-1}\|$. This leads to mini-batch requirements scaling as $\|\C\|^2\cdot\|\C^{-1}\|^2$, which can become arbitrarily large near non-diagonalizability, and introduces additional condition number dependence in $m$. Consequently, their guarantees degrade rapidly unless the mini-batch size is sufficiently large. These limitations stem from the reliance on product concentration results and diagonalization. In contrast, we avoid both and directly bound $\|\E[\Sigmab_t^\top \Sigmab_t]\|$.

\paragraph{Recursive approach to bound $\|\E[\Sigmab_t^\top\Sigmab_t]\|$.} Let $\bar\Sigmab_t := \E[\Sigmab_t] $ denote the matrix mean and $\M_t := \E[\Sigmab_t^\top\Sigmab_t] - \bar\Sigmab_t^\top\bar\Sigmab_t$ denote the matrix variance of $\Sigmab_t$. Moreover, let $\bar\Y = \E[\P\Y_t\P^\top]$ and $\Rb_t :=\P(\Y_t-\E[\Y_t])\P^\top$. Then, for any $p_{t-1}\geq\big\|\E[\Sigmab_{t-1}^\top\Sigmab_{t-1}]\big\|$, we have the following matrix recursion:
\begin{align}
    \M_t \preceq 
     p_{t-1}\cdot\E[\Rb_t^\top\Rb_t] + \bar\Y^\top\M_{t-1}\bar\Y.\label{e:pib_t_recursion_0}
\end{align}
Note that, $\bar\Y$ and $\bar\Sigmab_t$ are block-diagonal due to Lemma \ref{l:perm_0}. If one can show that $\E[\Rb_t^\top\Rb_t]$ is upper bounded by block-diagonal matrices, then via the matrix recursion, so is $\M_t$. We show this in the following lemma.

\begin{lemma}\label{l:second_moment_0}
 Let $\P$ be the permutation matrix from Lemma \ref{l:perm_0}. Then we have,
    \begin{align*}
        \E[\Rb_t^\top\Rb_t] &\preceq 4\cdot\diag(\W_1,\W_2,\cdots,\W_n),  \qquad \text{where} \qquad \W_i= \diag\Big(\frac{\lambda_i(1-\lambda_i)}{m}, \frac{\lambda_i(1-\lambda_i)}{m}\Big).
    \end{align*}
\end{lemma}
Using Lemma \ref{l:second_moment_0} and the repeated use of (\ref{e:pib_t_recursion_0}), it follows that both $\M_t$ and $\E[\Sigmab_t^\top\Sigmab_t]$ are also upper bounded by block-diagonal matrices. By taking a maximum over those blocks, we prove the following recursive expression for $\big\|\E[\Sigmab_t^\top\Sigmab_t]\big\|$.
\begin{lemma}\label{l:pt_recursion1_0}
For any $t\geq 1$, if there exist $p_0,p_1,p_2,\cdots,p_{t-1}$ such that $ \big\| \E[\Sigmab_{k}^\top\Sigmab_{k}]\big\| \leq p_k$ for $0\leq k \leq t-1$, then there exists $p_t$ such that for $q_i =4\lambda_i(1-\lambda_i)/m$,
     \begin{align*}
     \big\| \E[\Sigmab_t^\top\Sigmab_t]\big\| \leq   p_t \leq \max_{1\leq i \leq r}\Big(\big\|(\T_i^\top)^t\T_i^t\big\| + q_i\cdot\sum_{j=0}^{t-1}{p_j\cdot\big\|(\T_i^\top)^{t-1-j}\T_i^{t-1-j}\big\|}\Big).
    \end{align*}
\end{lemma}
\paragraph{Unfolding the recursion.} The next step is to unfold the recursive expression for $p_t$. By coupling with a simpler recursion, we show that, defining $\rho_t:=\max_{1\leq i\leq r}\big(|\gamma_{i}| + q_i \cdot\ell_i(t)\big)$,
\begin{align}
   p_{t}
    &\leq \rho_t^{t} + \sum_{j=0}^{t-1}{\rho_t^{j}\cdot\max_{1\leq i \leq r}\big\|(\T_i^\top)^{t-j}\T_i^{t-j}\big\|}, \label{e:gt_bound0}
\end{align}
for a carefully defined function $\ell_i(t)$, see \eqref{e:ell_i_def}. We can now use the bias analysis (Lemmas \ref{l:spec_radius_T_0} and \ref{l:schur_0}) to obtain tight bounds for the terms $\big\|(\T_i^\top)^{t-j}\T_i^{t-j}\big\|$. The final challenging task is to bound $\rho_t$.

\paragraph{Upper bounding $\rho_t$.} We provide a fine-grained analysis of $\rho_t$ by carefully analyzing $|\gamma_{i}| + q_i\cdot\ell_i(t)$ for every single $i \leq r$. Recall that, depending upon the interaction between $\lambda_i$ and $\beta$, the eigenvalue $\gamma_i$ can be real or complex (Lemma \ref{l:spec_radius_T_0}).

\paragraph{Case 1: $\lambda_i - \frac{(1-\beta)^2}{(1+\beta)^2}$ is sufficiently positive.} This is the easy case where $\gamma_i$ is complex, and we get
\begin{align}
    |\gamma_{i}| + q_i\cdot\ell_i(t) < 1-\frac{1}{32\phi}, \qquad \text{for} \ \phi \leq\frac{m}{C}.\label{e:rate_1}
\end{align}
\paragraph{Case 2: $\lambda_i - \frac{(1-\beta)^2}{(1+\beta)^2}$ is close to zero.} In this case, we enter the regime of near non-diagonalizability of $\T$. A crude analysis here would lead to a blow up factor in $m$. However, we are able to avoid this by exploiting the specific form of $\ell_i(t)$ and combine it with $\lambda_i \approx \frac{(1-\beta)^2}{(1+\beta)^2}$ to show that (\ref{e:rate_1}) still holds.

\paragraph{Case 3: $\lambda_i - \frac{(1-\beta)^2}{(1+\beta)^2}$ is sufficiently negative.}
Previous works such as \cite{bollapragada2022fast,pan2023accelerated, lee2022trajectory} select large values for the momentum parameter $\beta$ and for the mini-batch size $m$, because they enforce the condition that  $\lambda_i \geq \frac{(1-\beta)^2}{(1+\beta)^2}$. Our crucial insight is that, for smaller values of $\beta$, the quantity $\lambda_i - \frac{(1-\beta)^2}{(1+\beta)^2} $ can be negative and sufficiently bounded away from $0$ for some $i$. In fact, this is the case that determines the convergence rate for $\E\|\Deltab_t\|^2$ and allows us to handle SGD with classical momentum even under small mini-batches (see details in Appendix \ref{s:l2norm}). In particular, in this case we show that for $\phi \leq \frac{m}{C}$,
\begin{align}
    |\gamma_{i}| + q_i\cdot\ell_i(t) \leq  1-\frac{\phi\lambda_i}{4}. \label{e:rate_3}
\end{align}
Combining (\ref{e:rate_1}), (\ref{e:rate_3}) and maximizing over $i$, bounds $\rho_t$. Substituting in (\ref{e:gt_bound0}) and bounding $\big\|(\T_i^\top)^{t-j}\T_i^{t-j}\big\|$ bounds $p_t$. Noting the relation between $m$ and $\phi$ finishes the proof of Theorem~\ref{t:L2_iterate_0}.

\section{Experiments}\label{s:experiment}
For our experiments, we consider kernel ridge regression on four benchmark datasets sourced from Scikit-learn \cite{pedregosa2011scikit} and OpenML \cite{vanschoren2014openml} (more details and results in Appendix \ref{s:add_experiments}). For each dataset, we solve the resulting psd linear system $(\K + \lambda \I)\w = \y$, where $\K$ is the kernel matrix and $\lambda\ge 0$ is the ridge parameter,  using randomized block coordinate descent  \eqref{eq:cd} with uniform sampling as the base update. We then incorporate mini-batching by averaging $m$ independently generated update vectors in each iteration. Note that this mini-batching is separate from block sampling, which is less readily parallelizable due to the block inverse step. Each of the $m$ update vectors in the mini-batch is produced using an independently sampled random block. Finally, we consider  four different configurations of the algorithm, depending on the type of momentum acceleration: 

~\quad\textbf{1.~CD}: Vanilla block coordinate descent \eqref{eq:cd} using mini-batches of size $m$ and varying block sizes.

~\quad\textbf{2.~CDpp}: For comparison, we include a tuned state-of-the-art acceleration scheme for CD (i.e., not classical momentum) based on a Lyapunov-based approach with two auxiliary sequences \cite{gower2018accelerated,derezinski2025fine,derezinski2025randomized}.

~\quad\textbf{3.~CD+NAG }$(\beta = 1-1/m)$: CD with classical NAG momentum (i.e., $\omega=1$ in our framework), using a simple choice of momentum parameter motivated by our theory.

~\quad\textbf{4.~CD+NAG }$(\beta = \text{adaptive})$:  We implement a heuristic strategy that learns the near-optimal $\beta$ for any $m$ by executing an adaptive binary search procedure at runtime (see Appendix \ref{s:add_experiments}). 

\begin{figure}[t]
  \centering
\vspace{-2mm}
  \begin{subfigure}[t]{0.49\linewidth}
    \centering
    \includegraphics[width=\linewidth]{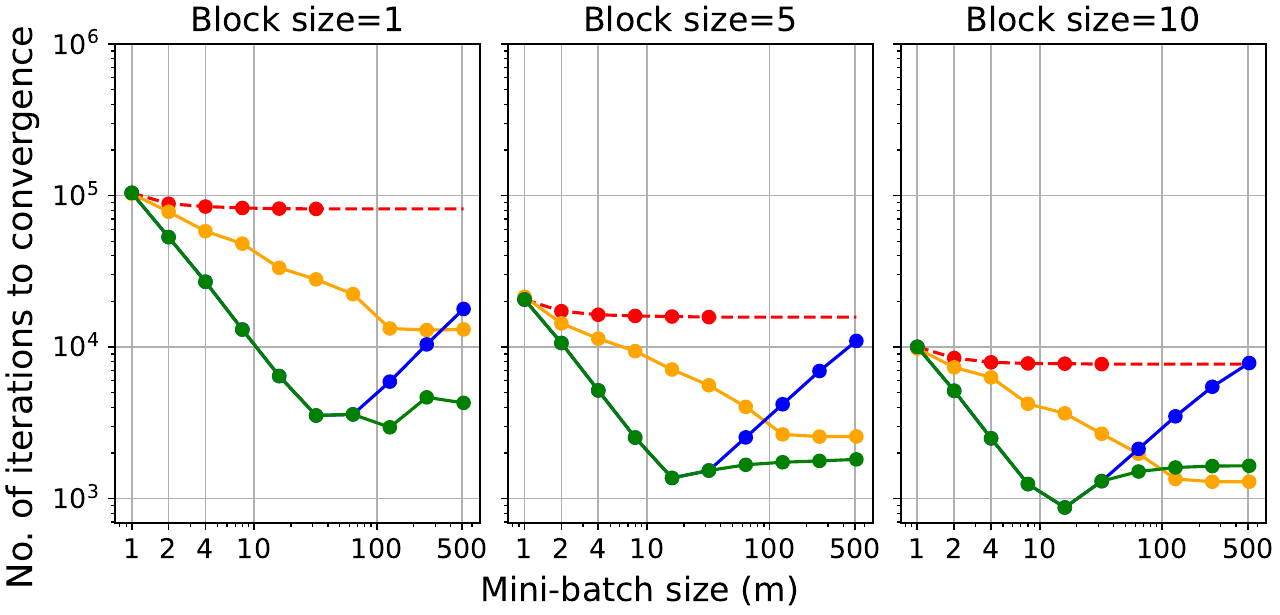}
    \caption{Covtype: iterations vs mini-batch size}
  \end{subfigure}
  \hfill
  \begin{subfigure}[t]{0.49\linewidth}
    \centering
    \includegraphics[width=\linewidth]{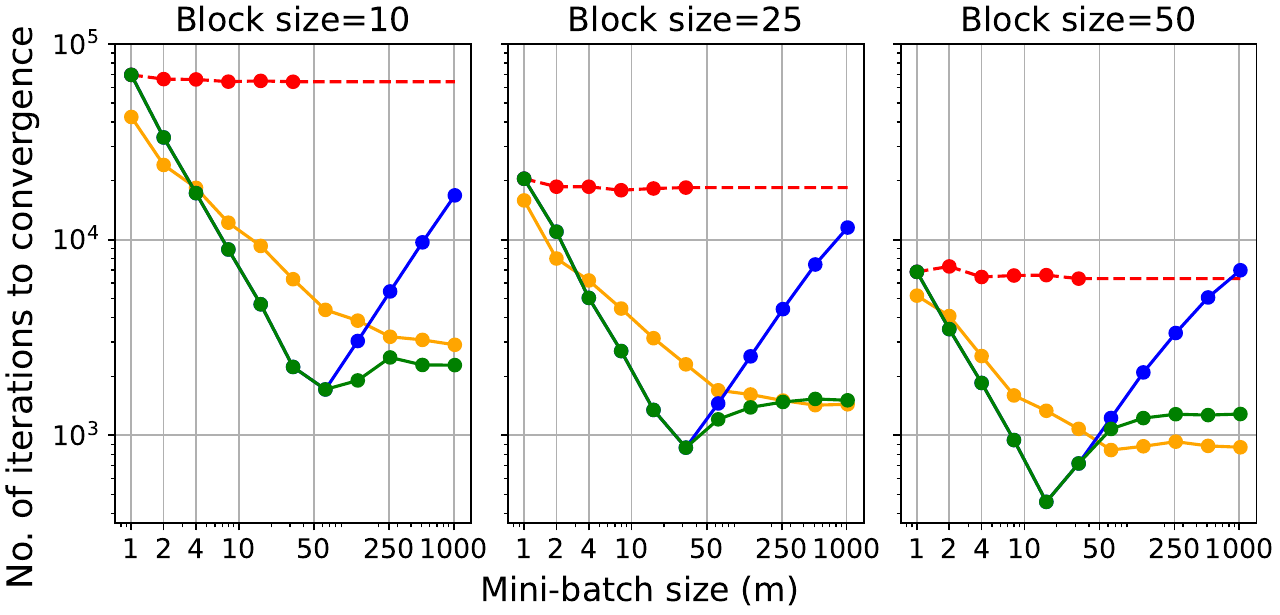}
    \caption{California Housing: iterations vs mini-batch size}
  \end{subfigure}

  \vspace{0.3em}

  \begin{subfigure}[t]{0.49\linewidth}
    \centering
    \includegraphics[width=\linewidth]{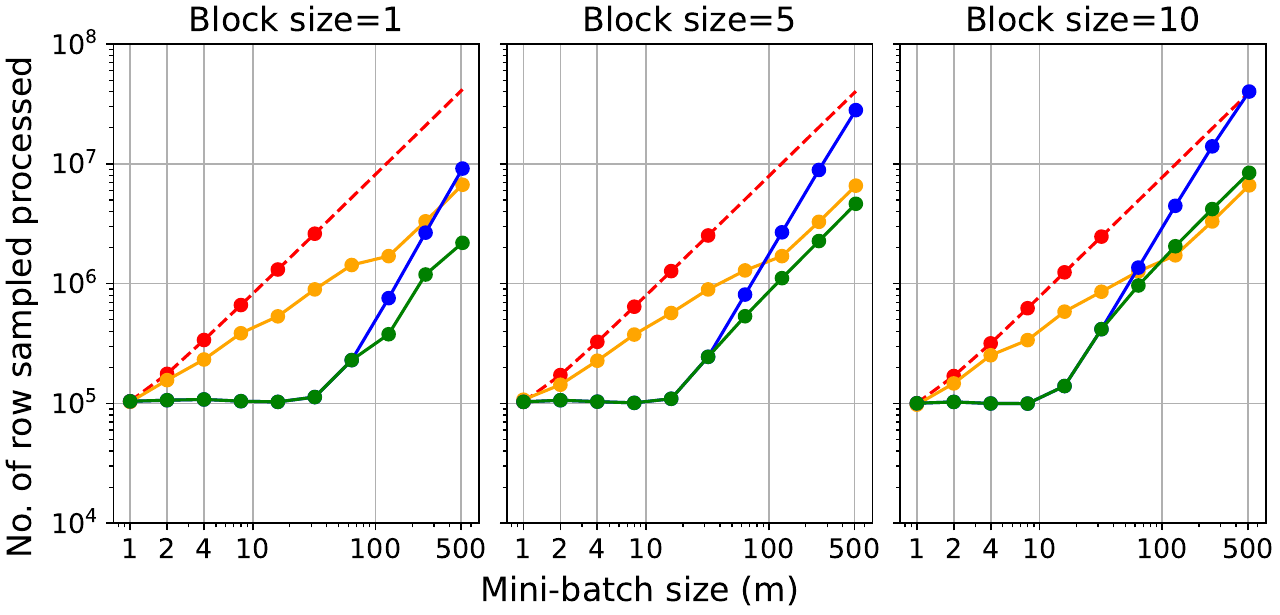}
    \caption{Covtype: row samples vs mini-batch size}
  \end{subfigure}
  \hfill
  \begin{subfigure}[t]{0.49\linewidth}
    \centering
    \includegraphics[width=\linewidth]{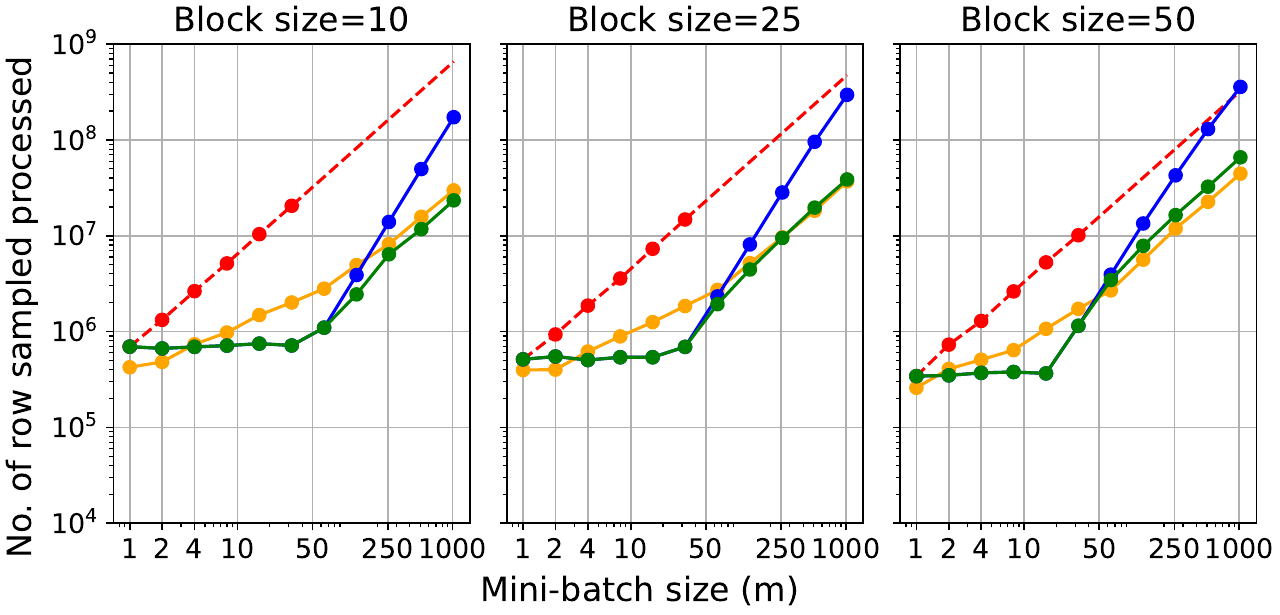}
    \caption{California Housing: row samples vs mini-batch size}
  \end{subfigure}

  \vspace{0.5em}

  \begin{minipage}{0.9\linewidth}
    \centering
    {\small
    \textcolor{myred}{\rule{1em}{1em}} CD \quad
    \textcolor{myorange}{\rule{1em}{1em}} CDpp \quad
    \textcolor{myblue}{\rule{1em}{1em}} CD+NAG $(\beta=1-1/m)$ \quad
    \textcolor{mygreen}{\rule{1em}{1em}}  CD+NAG $(\beta=\text{adaptive})$}
  \end{minipage}

  \caption{
  Convergence of block coordinate descent with NAG momentum and baselines.
  Top row: iterations to reach error $10^{-7}$.
  Bottom row: total work, measured by the number of sampled rows.
  }\vspace{-1mm}
  \label{fig:california_covtype}
\end{figure}

\begin{figure}[t]
  \centering

  \begin{subfigure}[t]{0.49\linewidth}
    \centering
    \includegraphics[width=\linewidth]{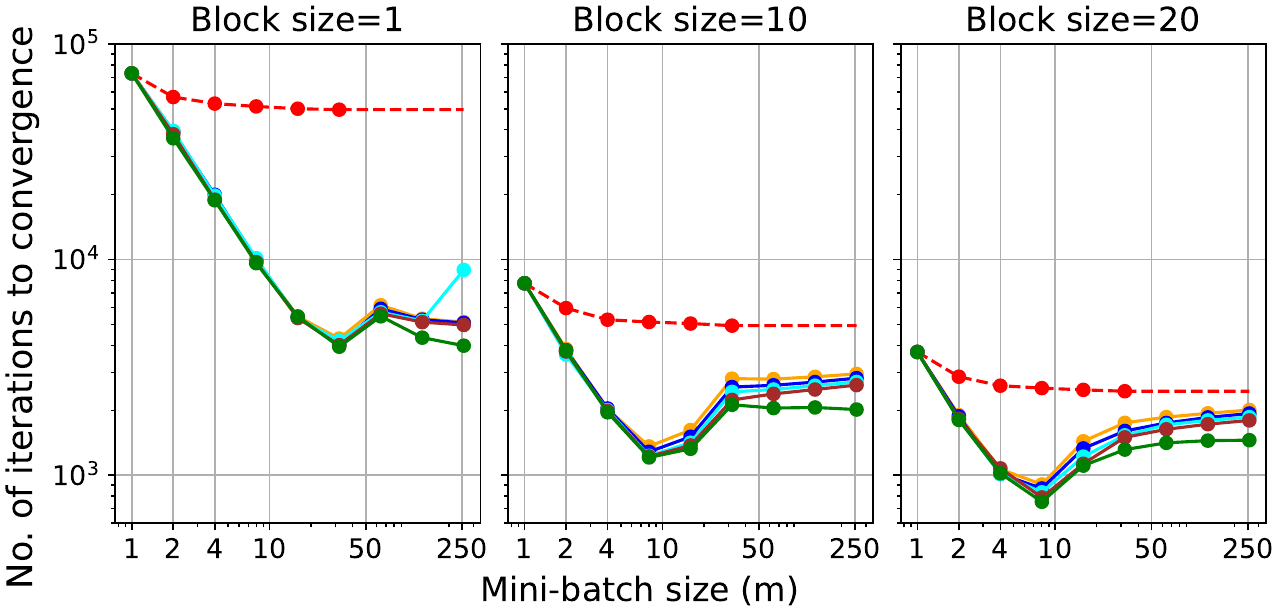}
    \caption{Abalone: iterations vs mini-batch size}
  \end{subfigure}
  \hfill
  \begin{subfigure}[t]{0.49\linewidth}
    \centering
    \includegraphics[width=\linewidth]{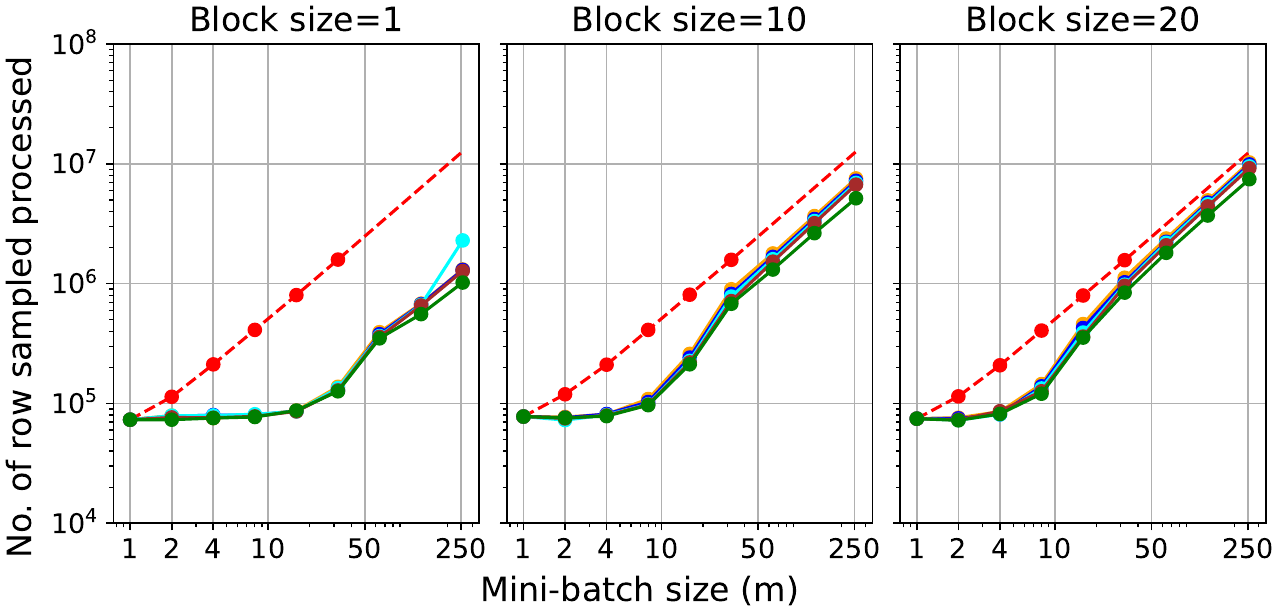}
    \caption{Abalone: row samples vs mini-batch size}
  \end{subfigure}

  \vspace{0.5em}

  \begin{minipage}{1\linewidth}
    \centering
    \scriptsize
    \textcolor{myred}{\rule{1em}{1em}} CD \hspace{0.6em}
    \textcolor{myorange}{\rule{1em}{1em}} CDM : $\omega=0$ (HBM) \hspace{0.6em}
    \textcolor{myblue}{\rule{1em}{1em}} CDM : $\omega=0.25$ \hspace{0.6em}
    \textcolor{mycyan}{\rule{1em}{1em}} CDM : $\omega=0.50$ \hspace{0.6em}
    \textcolor{mybrown}{\rule{1em}{1em}} CDM : $\omega=0.75$ \hspace{0.6em}
    \textcolor{mygreen}{\rule{1em}{1em}} CDM : $\omega=1$ (NAG) \hspace{0.6em}
  \end{minipage}

  \caption{
  Convergence behavior for CDM with different $\omega$ across block sizes.
  Left column: iterations to reach error $10^{-10}$.
  Right column: total work, measured by the number of sampled rows.
  }
  \label{fig:abalone_all_omega}
\end{figure}
We also separately compare several variants of classical momentum by interpolating between \textbf{CD+HBM} ($\omega=0$) and \textbf{CD+NAG} ($\omega=1$), see Fig.~\ref{fig:abalone_all_omega}. Although our theory (Theorem \ref{t:L2_iterate_0}) guarantees identical convergence rates for all $\omega \in [0,1]$, empirical performance can depend on the choice of $\omega$, with larger values often yielding improved results. This observation is also noted in the remark following Lemma \ref{l:spec_radius_T}. Furthermore, as shown in Fig.~\ref{fig:abalone_all_omega}, all variants of classical momentum perform very similarly; nevertheless, $\omega = 1$ consistently achieves the best results, which is why we chose NAG momentum in our main experiments.

In Fig.~\ref{fig:california_covtype}(a)-(b) we can see that, as expected, mini-batching by itself does not improve the convergence of \textbf{CD}. (Due to this phenomenon, we ran \textbf{CD} only up to mini-batch size $m=32$, and extrapolated for larger values of $m$.) However, incorporating momentum into the update unlocks significant speed-up from mini-batching. Remarkably, while the specialized acceleration scheme of \textbf{CDpp} tends to attain better acceleration than \textbf{CD+NAG} for $m=1$, the classical momentum does a better job of exploiting large mini-batches. This can be seen even more clearly in Fig.~\ref{fig:california_covtype}(c)-(d), which shows total work performed by the algorithm. For \textbf{CD+NAG}, the total work stays constant as we increase $m$, until reaching a saturation point, which allows for perfect parallelization of mini-batches. Notably, this perfect parallelization phenomenon is not observed for \textbf{CDpp}.

Next, comparing \textbf{CD+NAG} with $\beta=1-1/m$ to its adaptive variant, we observe that in the perfect parallelization regime the simple choice of momentum parameter is essentially optimal. For larger $m$, as suggested by our theory, $\beta$ should be capped at $1-\Theta(1/\sqrt\kappa)$, which is why the adaptive selection procedure becomes more effective.

\section{Conclusions}
We develop a theoretical framework for quantifying the effect of classical momentum schemes like heavy ball on the convergence of mini-batch SGD. Our results demonstrate that momentum  enables perfect parallelization of mini-batch computations, which is confirmed by numerical~experiments.

\subsection*{Acknowledgments}
This work was supported in part by NSF CAREER Grant CCF-233865 and a Google ML and Systems Junior Faculty Award. The work was done in part while MD was visiting the Simons Institute for the Theory of Computing.

\bibliography{references}
\bibliographystyle{plain}

\appendix
\section{Convergence analysis for expected iterates}
We start by recalling the matrix transition rule (\ref{e:stoc_proc_two_step}),
 \begin{align*}
        \begin{bmatrix}
            \V^\top\Deltab_{t+1} \\ \V^\top\Deltab_t
        \end{bmatrix} = \underbrace{\begin{bmatrix} (1+\beta)\I-(1+\beta\omega)\V^\top\Pib_t\V & -\beta\cdot(\I-\omega\V^\top\Pib_{t-1}\V)\\
        \I & \zero\end{bmatrix}}_{=\Y_t}\cdot \begin{bmatrix}
            \V^\top\Deltab_{t} \\ \V^\top\Deltab_{t-1}
        \end{bmatrix}
    \end{align*}
 where $\omega \in [0,1]$ and $\Deltab_{-1} =0$. As $\text{range}(\Pib_t) \subset \text{range}(\bar\Pib)$, we note that if $\Deltab_0, \Deltab_1 \in \text{range}(\bar\Pib)$, then for all $t$, $\Deltab_t \in  \text{range}(\bar\Pib)$. In this section, we analyze the update rule (\ref{e:stoc_proc_two_step}) in expectation.
 Let $\T$ denote $\E\Y_t$. We have,
    \begin{align*}
         \T = \begin{bmatrix} (1+\beta)\I - (1+\beta\omega)\V^\top\bar\Pib\V & -\beta\cdot(\I-\omega\V^\top\bar\Pib\V)\\
        \I &\zero\end{bmatrix}.
    \end{align*}
     It is useful to transform $\T$ into a block-diagonal form. The following lemma guarantees that a permutation matrix $\P$ exists such that $\T$ can be transformed into a block-diagonal matrix.
    \begin{lemma}[Block-diagonalization via permutation]\label{l:perm}
        There exists a permutation matrix $\P$ such that,
        \begin{align*}
           \P\T\P^\top = \diag(\T_1,\T_2,\cdots,\T_n),
        \end{align*}
        where $\T_i =  \begin{bmatrix}
            (1+\beta)-(1+\beta\omega)\lambda_i & -\beta(1-\omega\lambda_i) \\
            1&0
        \end{bmatrix},$ and 
        \begin{align*}
            \P \begin{bmatrix} \V^\top\Deltab_{t+1} \\
            \V^\top\Deltab_t\end{bmatrix} = \begin{bmatrix}
            \v_1^\top\Deltab_{t+1} \\
            \v_1^\top\Deltab_{t}\\
            \cdots\\
            \v_d^\top\Deltab_{t+1}\\
            \v_d^\top\Deltab_t
                \end{bmatrix}.
        \end{align*}
        
    \end{lemma}
The resulting block-diagonal structure helps us throughout the analysis. As $(\P\T\P^\top)^t = \P\T^t\P^\top$ and $\|\P\T^t\P^\top\| = \|\T^t\|$, we need to upper bound $\|(\P\T\P^\top)^t\|$. Due to the block-diagonal form of $\P\T\P^\top$, we have $\|(\P\T\P^\top)^t\| = \max_{i}\|\T_i^t\|$. For any fixed $i$, let $\gamma_{i1},\gamma_{i2} \in \mathbb{C}$ denote the two eigenvalues of $\T_i$. We always take $\gamma_{i1}$ to be the eigenvalue with larger magnitude. In the next lemma, we provide a tight upper bound on the magnitude of eigenvalues of $\T_i$.

    \begin{lemma}[Bounding spectral radius of transition matrix]\label{l:spec_radius_T}
      Let $\omega \in[0,1]$ and $\beta = 1-\frac{1}{\phi}$ for some $\phi \geq 2$. Then, the following holds:
      
       \vspace{2mm}
       \noindent If $\T_i$ has complex eigenvalues then,
       \begin{align*}
           |\gamma_{i2}|^2 = |\gamma_{i1}|^2 = \Big(1-\frac{1}{\phi}\Big)(1-\omega\lambda_i).
       \end{align*}
      
       \noindent If $\T_i$ has real and equal eigenvalues then,
       \begin{align*}
           |\gamma_{i2}|^2 = |\gamma_{i1}|^2 = \Big(1-\frac{1}{\phi}\Big)(1-\omega\lambda_i).
       \end{align*}

       \noindent If $\T_i$ has real and distinct eigenvalues then,
       \begin{align*}
           |\gamma_{i2}|^2 < |\gamma_{i1}|^2 < \Big(1-\frac{\phi\lambda_i}{2}\Big)^2(1-\omega\lambda_i)^2.
       \end{align*}
    \end{lemma}
    \begin{proof}
    We have,
    \begin{align*}
        \T_i = \begin{bmatrix}
            (1+\beta)-(1+\beta\omega)\lambda_i & -\beta(1-\omega\lambda_i) \\
            1&0
        \end{bmatrix}
    \end{align*}
   It is easy to verify that the eigenvalues of $\T_i$, i.e., $\gamma_{i1}$ and $\gamma_{i2}$ are determined by the following expression.
        \begin{align*}
            (\gamma_{i1},\gamma_{i2}) = \frac{1}{2}\Big((1-\lambda_i) + \beta(1-\omega\lambda_i) \pm \sqrt{\big((1-\lambda_i) + \beta(1-\omega\lambda_i)\big)^2-4\beta(1-\omega\lambda_i)}\Big).
        \end{align*}
        The expression inside the square root can be negative, zero, or positive. So we consider three cases:
        \begin{enumerate}
        \item $\big((1-\lambda_i) + \beta(1-\omega\lambda_i)\big)^2-4\beta(1-\omega\lambda_i)<0$: In this case, both $\gamma_{i1}$ and $\gamma_{i2} $ are complex and have the same magnitude. We get,
        \begin{align}
           |\gamma_{i2}|^2 = |\gamma_{i1}|^2 = \beta(1-\omega\lambda_i) = \Big(1-\frac{1}{\phi}\Big)(1-\omega\lambda_i).\label{e:exp_nesterov_1}
        \end{align}
        \item $\big((1-\lambda_i) + \beta(1-\omega\lambda_i)\big)^2-4\beta(1-\omega\lambda_i)=0$. In this case, $\gamma_{i1}^2 =\gamma_{i2}^2 = \frac{1}{4}\Big((1-\lambda_i) + \beta(1-\omega\lambda_i)\Big)^2 = \beta(1-\omega\lambda_i)$  . Therefore,
        \begin{align}
            \gamma_{i2}^2 = \gamma_{i1}^2 = \Big(1-\frac{1}{\phi}\Big)(1-\omega\lambda_i).\label{e:exp_nesterov_2}
        \end{align}
       
        \item $\big((1-\lambda_i) + \beta(1-\omega\lambda_i)\big)^2-4\beta(1-\omega\lambda_i)>0$. The larger eigenvalue is given as,
        \begin{align}
            \gamma_{i1} &=\frac{1}{2}\Big((1-\lambda_i) + \beta(1-\omega\lambda_i) \pm \sqrt{\big((1-\lambda_i) + \beta(1-\omega\lambda_i)\big)^2-4\beta(1-\omega\lambda_i)}\Big). \label{e:exp_nesterov_3}
        \end{align}
       Let $u_i = \frac{1-\lambda_i}{1-\omega\lambda_i}$,
\begin{align*}
    \gamma_{i1} &= \frac{1}{2}\Big(u_i +\beta + \sqrt{\big(u_i+\beta\big)^2 - \frac{4\beta}{1-\omega\lambda_i}}\Big)(1-\omega\lambda_i)\\
    &=\frac{1}{2}\Big(u_i+\beta +\sqrt{u_i^2+\beta^2 +\frac{2\beta(1-\lambda_i)}{1-\omega\lambda_i} - \frac{4\beta}{1-\omega\lambda_i}}\Big)(1-\omega\lambda_i)\\
    &=\frac{1}{2}\Big(u_i+\beta + \sqrt{u_i^2 +\beta^2 - \frac{2\beta(1+\lambda_i)}{1-\omega\lambda_i}}\Big)(1-\omega\lambda_i)\\
    &= \frac{1}{2}\Big(u_i+\beta +\sqrt{u_i^2+\beta^2 - \frac{2\beta(1-\lambda_i)}{1-\omega\lambda_i} - \frac{4\beta\lambda_i}{1-\omega\lambda_i}}\Big)(1-\omega\lambda_i)\\
    &=\frac{1}{2}\Big(u_i+\beta + \sqrt{(u_i-\beta)^2 - \frac{4\beta\lambda_i}{1-\omega\lambda_i}}\Big)(1-\omega\lambda_i)
\end{align*}
        Let $x=u_i-\beta$, $y=\sqrt{\frac{4\beta\lambda_i}{1-\omega\lambda_i}}$, then we have $x=\frac{yr}{2}$, where $r=\frac{2x}{y}>2$. Now, if $x,y \in \R$, $x>y>0$ and $x\leq \frac{yr}{2}$ for some $r> 2$, then $\sqrt{x^2-y^2} \leq x-\frac{y}{r}$. Therefore,
        \begin{align*}
            \gamma_{i1} &< \frac{1}{2}\Big(u_i+\beta + u_i-\beta - \frac{y}{r}\Big)(1-\omega\lambda_i) \nonumber\\
            &= \Big(u_i-\frac{y^2}{4x}\Big)(1-\omega\lambda_i) \nonumber\\
          &=\Big(1-\frac{y^2}{4(1-\beta)}\Big)(1-\omega\lambda_i).
        \end{align*}
        The last inequality holds as $u_i \leq 1$ and $x \leq 1-\beta$. Therefore, we get
        \begin{align*}
            \gamma_{i1} < \Big(1-\frac{\phi\beta\lambda_i}{1-\omega\lambda_i}\Big)(1-\omega\lambda_i)
        \end{align*}
        Now for $\beta \geq 0.5$, we get
        \begin{align}
            \gamma_{i1} < \Big(1-\frac{\phi\lambda_i}{2}\Big)(1-\omega\lambda_i) \label{e:exp_nesterov_4}
        \end{align}
      \end{enumerate}
       Combining the relations (\ref{e:exp_nesterov_1}), (\ref{e:exp_nesterov_2}) and (\ref{e:exp_nesterov_4}) we finish the proof. 
    \end{proof}
    \begin{remark}\label{r:NAG_over_HBM}
        Lemma \ref{l:spec_radius_T} implies that NAG momentum can perform slightly better than HBM as $\omega=1$ for NAG and $\omega=0$ for HBM. We demonstrate this phenomenon empirically in Fig~\ref{fig:abalone_all_omega}, where we compare iteration complexity of block coordinate descent with momentum for various values of $\omega$.
    \end{remark}

     We are now ready for the analysis of expected iterates. As $\Deltab_t \in \text{range}(\bar\Pib)$ for all $t$, if for some $i$, $\v_i \in \text{Null}(\bar\Pib)$, then $\v_i^\top\Deltab_t=0$. Let $r$ denote the rank of $\bar\Pib$ and note that $\v_i^\top\Deltab_{t} =0$ for all $t$ and $i >r$. Let $\phi$ be any fixed constant satisfying $2\leq  \phi \leq \min\big\{\frac{m}{C}, \frac{1}{3\sqrt{\lambda_r}}\big\}$, where $C$ is some absolute constant (to be determined later) and $m$ is as it appears in Proposition \ref{p:mini-batching}. Let the momentum parameter $\beta=1-\frac{1}{\phi}$. We prove the following result:
    \begin{theorem}\label{t:exp_iterate}
     Let $\omega \in [0,1]$. Then, the matrix transition rule (\ref{e:stoc_proc_two_step}) satisfies,
     \begin{align*}
         \|\E\Deltab_{t+1}\|^2 \leq 17\cdot(3t+2)^2\cdot\rho^{t-1}.\|\Deltab_0\|^2,
     \end{align*} 
     where $\rho = \big(1-\frac{\phi\lambda_r}{2}\big)(1-\omega\lambda_r)$.
    \end{theorem}
    We will need the following auxiliary lemmas to finish the proof of Theorem \ref{t:exp_iterate}.
    \begin{lemma}[Two-step matrix transition rule]\label{l:exp_mat}
    We have
     \begin{align*}
            \|\E\Deltab_{t+1}\|^2 \leq 17\cdot \max_{1\leq i \leq r}\|\T_i^t\|\cdot\|\Deltab_0\|^2.
        \end{align*}
    \end{lemma}
    \begin{proof}
        Let $t\geq 1$. The update rule (\ref{e:stoc_proc_two_step}) satisfies,
     \begin{align*}
            \E\P\begin{bmatrix}
                \V^\top\Deltab_{t+1} \\ \V^\top\Deltab_t
            \end{bmatrix} = \P\T\P^\top\cdot \E\P\begin{bmatrix}
                \V^\top\Deltab_{t} \\ \V^\top\Deltab_{t-1}
            \end{bmatrix}.
        \end{align*}
        Due to the block diagonal structure of $\P\T\P^\top$, we get
        \begin{align*}
            \E\begin{bmatrix}
                \v_i^\top\Deltab_{t+1}\\
                \v_i^\top\Deltab_t
            \end{bmatrix} = \T_i\cdot\E\begin{bmatrix}
                \v_i^\top\Deltab_{t}\\
                \v_i^\top\Deltab_{t-1} 
            \end{bmatrix}
        \end{align*}
        Recursively applying the above relation we get
 \begin{align*}
            \E\begin{bmatrix}
                \v_i^\top\Deltab_{t+1}\\
                \v_i^\top\Deltab_t
            \end{bmatrix} = \T_i^{t}\cdot\E\begin{bmatrix}
                \v_i^\top\Deltab_{t}\\
                \v_i^\top\Deltab_{t-1} 
            \end{bmatrix}
            \end{align*}
Therefore,
\begin{align}
    \|\E\v_i^\top\Deltab_{t+1}\|^2 \leq \|\E\v_i^\top\Deltab_{t+1}\|^2 + \|\E\v_i^\top\Deltab_t\|^2 \leq \|\T_i^t\|^2\cdot\big(\|\E\v_i^\top\Deltab_1\|^2 + \|\v_i^\top\Deltab_0\|^2\big). \label{e:exp_mat_1}
\end{align}
As by notation $\Deltab_{-1}=0$, we have
\begin{align*}
    &\E\v_i^\top\Deltab_1 = \big(1+\beta - (1+\beta\omega)\lambda_i\big)\v_i^\top\Deltab_0\\
    &\|\E\v_i^\top\Deltab_1\|^2 \leq 16\|\v_i^\top\Deltab_0\|^2.
\end{align*}
Substituting the bound for $\|\E\v_i^\top\Deltab_1\|^2$ in (\ref{e:exp_mat_1}) we get
\begin{align*}
     \|\E\v_i^\top\Deltab_{t+1}\|^2 \leq 17\|\T_i^t\|^2\cdot\|\v_i^\top\Deltab_0\|^2.
\end{align*}
Now summing over $i$ and noting that $\v_i^\top\Deltab_{t} =0$ for all $i >r$, we get
\begin{align*}
    \|\E\V^\top\Deltab_{t+1}\|^2 \leq 17\cdot \max_{1\leq i \leq r}\|\T_i^t\|\cdot\|\V^\top\Deltab_0\|^2.
\end{align*}
Nothing that $\V$ has orthonormal columns finishes the proof.
    \end{proof}
To complete the proof of Theorem \ref{t:exp_iterate}, we need to bound the spectral norm of $\T_i^t$. Note that Lemma~\ref{l:spec_radius_T} bounds the spectral radius of $\T_i^t$, not the spectral norm. For bounding the spectral norm, we use the following fundamental matrix decomposition result. For any complex valued matrix $\U$, let $\U^{H}$ denote the conjugate transpose of $\U$.
\begin{lemma}[Schur's decomposition, Theorem 2.3.1 in \cite{plemmons1988matrix}]\label{l:schur}
    For every $i$, there exists a unitary matrix $\U_i$ such that
    \begin{align*}
        \U^{H}_i\T_i\U_i = \begin{bmatrix}
            \gamma_{i1} &x_i\\
            0 &\gamma_{i2}
        \end{bmatrix},
    \end{align*}
    where $\gamma_{i1}$ and $\gamma_{i2}$ are the two eigenvalues of $\T_i$ and $x_i$ is a constant such that $1 \leq |x_i| \leq 3$. The unitary matrix, $\U_i$ is explicitly given as
    \begin{align*}
        \U_i = \frac{1}{\sqrt{1+|\gamma_{i1}|^2}}\cdot\begin{bmatrix} \gamma_{i1} &-1\\
        1 & \bar\gamma_{i1}\end{bmatrix},
    \end{align*}
    where $\bar\gamma_{i1}$ denotes complex conjugate of $\gamma_{i1}$. Furthermore, for any $k\geq 1$,
    \begin{align*}
         \U^{H}_i\T_i^k\U_i = \begin{bmatrix} \gamma_{i1}^k & x_i\sum_{j=0}^{k-1}{\gamma_{i2}^{j}\gamma_{i1}^{k-1-j}}\\
         0 &\gamma_{i2}^k\end{bmatrix}
    \end{align*}
\end{lemma}
\begin{corollary}
    If $\gamma_{i1} \ne \gamma_{i2}$, then
     $ x_i\sum_{j=0}^{k-1}{\gamma_{i2}^{j}\gamma_{i1}^{k-1-j}} = x_i \gamma_{i1}^{k-1}\cdot \frac{1-(\gamma_{i2}/\gamma_{i1})^k}{1-(\gamma_{i2}/\gamma_{i1})}$, implying \begin{align*}
         |x_i\sum_{j=0}^{k-1}{\gamma_{i2}^{j}\gamma_{i1}^{k-1-j}}| < |x_i|\cdot|\gamma_{i1}|^k\cdot\frac{3}{|\gamma_{i1}-\gamma_{i2}|}.
     \end{align*}
\end{corollary}
\begin{corollary}
    If $\gamma_{i1} = \gamma_{i2}$, then
  $ x_i\sum_{j=0}^{k-1}{\gamma_{i2}^{j}\gamma_{i1}^{k-1-j}} = x_{i}\gamma_{i1}^{k-1}\cdot k$, implying,
  \begin{align*}
        |x_i\sum_{j=0}^{k-1}{\gamma_{i2}^{j}\gamma_{i1}^{k-1-j}}| < |x_i|\cdot|\gamma_{i1}|^k\cdot\frac{k}{|\gamma_{i1}|}.
  \end{align*}
\end{corollary}
\begin{corollary}\label{c:spec_norm_ti}
     For any $i$,
   \begin{align*}
       \|\T_i^k\| = \|\U_i^H\T_i^k\U_i\| \leq (3k+2)\cdot |\gamma_{i1}|^{k-1}.
   \end{align*}
\end{corollary}
\paragraph{Finishing proof of Theorem \ref{t:exp_iterate}.} Using Corollary \ref{c:spec_norm_ti} we get $\|\T_i^t\|^2 \leq (3t+2)^2\cdot|\gamma_{i1}|^{2(t-1)}$. Furthermore, if $\lambda_r$ is such that $\big((1-\lambda_r) + \beta(1-\omega\lambda_r)\big)^2-4\beta(1-\omega\lambda_r)<0$, then we have $|\gamma_{i1}|^2 \leq \Big(1-\frac{1}{2\phi}\Big)(1-\omega\lambda_r)$ for  $i \leq r$. On the other hand if, $\big((1-\lambda_i) + \beta(1-\omega\lambda_i)\big)^2-4\beta(1-\omega\lambda_i)<0$, then $|\gamma_{i1}|^2 < \Big(1-\frac{\phi\lambda_r}{2}\Big)(1-\omega\lambda_r)$ for $i \leq r$. Now the statement follows by noting that $\phi < \frac{1}{\sqrt{\lambda_r}}$.
\section{Convergence in $\ell_2$-norm with mini-batch averaging}\label{s:l2norm}
In this section, we aim to recover a result similar to Theorem \ref{t:exp_iterate} but in $\ell_2$ norm. Recall the transformed update rule,
\begin{align*}
    \P\begin{bmatrix}
       \V^\top \Deltab_{t+1} \\
       \V^\top \Deltab_t 
    \end{bmatrix}&=\P\Y_t\P^\top \cdot\P\begin{bmatrix}
       \V^\top \Deltab_{t} \\
       \V^\top \Deltab_{t-1}
    \end{bmatrix},
\end{align*}
As $\v_i^\top\Deltab_t =0$ for $i>r$, we can restrict the dynamics to the top-$r$ dimensional space. Let $\Q \in \R^{2r\times 2d}$ consist of first $2r$ rows of $2d \times 2d$ identity matrix. We get,
\begin{align*}
    \Q \P\begin{bmatrix}
       \V^\top \Deltab_{t+1} \\
       \V^\top \Deltab_t 
    \end{bmatrix} = \Q\P\Y_t^\top\P^\top\Q^\top\cdot\Q\P\begin{bmatrix}
       \V^\top \Deltab_{t} \\
       \V^\top \Deltab_{t-1}
    \end{bmatrix}
\end{align*}
Recursively we get,
\begin{align*}
    \Q\P\begin{bmatrix}
        \V^\top\Deltab_{t+1} \\
        \V^\top\Deltab_t 
    \end{bmatrix} = \Big(\prod_{j=0}^{j={t-1}}{\Q\P\Y_{t-j}\P^\top\Q^\top}\Big)\cdot\Q\P\begin{bmatrix}
       \V^\top \Deltab_{1} \\
       \V^\top \Deltab_0 
    \end{bmatrix}.
\end{align*}
Therefore,
\begin{align*}
    \E\|\Deltab_{t+1}\|^2 \leq 17\cdot\Big\|\E\Big[\Big(\prod_{j=1}^{t}{\Q\P\Y_j^\top\P^\top\Q^\top}\Big)\Big(\prod_{j=0}^{t-1}{\Q\P\Y_{t-j}\P^\top\Q^\top}\Big)\Big]\Big\|\cdot \|\Deltab_0\|^2.
\end{align*}
Our task is to bound $\Big\|\E\Big[\Big(\prod_{j=1}^{t}{\Q\P\Y_j^\top\P^\top\Q^\top}\Big)\Big(\prod_{j=0}^{t-1}{\Q\P\Y_{t-j}\P^\top\Q^\top}\Big)\Big]\Big\|$.
We prove the following theorem, which is only a restatement of Theorem \ref{t:L2_iterate_0}.
\begin{theorem}[General result]\label{t:L2_iterate}
Let $\phi$ satisfy $2\leq  \phi \leq \min\big\{\frac{m}{C}, \frac{1}{3\sqrt{\lambda_r}}\big\}$ for $C>76800$.  Then, update rule (\ref{e:stoc_proc_two_step}) with $\beta= 1-\frac{1}{\phi}$ satisfies,
    \begin{align*}
        \E\|\Deltab_{t}\|^2  \leq 170t^3\cdot\rho^{t-1}\cdot\|\Deltab_0\|^2,
    \end{align*}
    where $\rho = 1- \frac{\phi\lambda_r}{4}$.
\end{theorem}

Note that Theorem \ref{t:L2_iterate_0} follows from Theorem \ref{t:L2_iterate} by taking $c_1 =2\cdot 76800$, $c_2 = c_1/3$.

\subsection{Proof of Theorem \ref{t:L2_iterate}}
We prove a series of results, building towards the proof of Theorem \ref{t:L2_iterate}. Let $\Z_t = \V^\top\Pib_t\V$, and $\bar\Z=\E\Z_t$. Note that $\bar\Z =\Lambdab$, where $\Lambdab = \diag(\lambda_1,\lambda_2,\cdots\lambda_d)$. We have,
    \begin{align*}
        \E(\Z_t-\bar\Z)^2 = \E(\Z_t^2)-\bar\Z^2 &= \E(\Z_t^2) - \Lambdab^2\\
        &=\V^\top\E(\Pib_t^2)\V-\Lambdab^2
    \end{align*}
    As by Definition \ref{d:stochastic-contraction}, $\Pib_t \preceq \I$, we always have $\E(\Pib_t^2) \preceq \bar\Pib$, implying $\E(\Z_t-\bar\Z)^2 \preceq \Lambdab-\Lambdab^2$. However, upper bounding $\E(\Pib_t^2)$ by $\E\Pib_t$ can be quite suboptimal, for instance, when $\Pib_t$ is itself an average of multiple independent contractions. For proving a tighter bound on the second moment of $\Z_t$, we use Proposition \ref{p:mini-batching}. Due to Proposition \ref{p:mini-batching}, we have,
    \begin{align}
        \E\big(\Z_t-\bar\Z\big)^2 = \V^\top\E\big(\Pib_t-\bar\Pib\big)^2\V \preceq \frac{1}{m}\cdot\V^\top\bar\Pib(\I-\bar\Pib)\V = \frac{1}{m}\cdot\Lambdab(\I-\Lambdab) \label{e:variance}
    \end{align}
    
\begin{lemma}\label{l:second_moment}
 Let $\P$ be the permutation matrix as in Lemma \ref{l:perm}. Then, we have,
    \begin{align*}
        \P\cdot\E\big[(\Y_t-\E\Y_t)^\top(\Y_t-\E\Y_t)\big]\cdot\P^\top &\preceq 4\cdot\diag(\W_1,\W_2,\cdots,\W_n),
    \end{align*}
   where $\W_i= \diag\Big(\frac{\lambda_i(1-\lambda_i)}{m}, \frac{\omega^2\lambda_i(1-\lambda_i)}{m}\Big)$.
\end{lemma}
\begin{proof}
We have $\Z_t = \V^\top\Pib_t\V$, $\Z_{t-1} = \V^\top\Pib_{t-1}\V$ and $\bar\Z=\E\Z_t = \E\Z_{t-1}$. We have,
    \begin{align*}
        \Y_t = \begin{bmatrix}
            (1+\beta)\I - (1+\beta\omega)\Z_t &-\beta(\I-\omega\Z_{t-1})\\
            \I&\zero
        \end{bmatrix}, \ \ \E\Y_t =\begin{bmatrix}
            (1+\beta)\I - (1+\beta\omega)\bar\Z &-\beta(\I-\omega\bar\Z)\\
            \I&\zero
        \end{bmatrix},
    \end{align*}
    and,
    \begin{align*}
       \E[ (\Y_t-\E\Y_t)^\top(\Y_t-\E\Y_t)] &= \E\begin{bmatrix}
            (1+\beta\omega)\big(\bar\Z-\Z_t) & \zero\\
            -\omega\beta(\bar\Z-\Z_{t-1}) &\zero
        \end{bmatrix}\cdot\begin{bmatrix}
            (1+\beta\omega)\big(\bar\Z-\Z_t) & -\omega\beta(\bar\Z-\Z_{t-1})\\
            \zero &\zero
        \end{bmatrix}\\
        &=\E\begin{bmatrix}
             (1+\beta\omega)^2(\Z_t-\bar\Z)^2 & -\omega\beta(1+\beta\omega)(\Z_t-\bar\Z)(\Z_{t-1}-\bar\Z) \\
             -\omega\beta(1+\beta\omega)(\Z_{t-1}-\bar\Z)(\Z_t-\bar\Z) & \omega^2\beta^2(\Z_{t-1}-\bar\Z)^2
             \end{bmatrix}\\
        &=\begin{bmatrix}
            (1+\beta\omega)^2\cdot\E\big(\Z_t-\bar\Z\big)^2 &\zero\\
            \zero & \omega^2\beta^2\cdot\E(\Z_{t-1}-\bar\Z)^2
        \end{bmatrix}.
    \end{align*}
Using relation (\ref{e:variance}) we get,
\begin{align*}
     \E[ (\Y_t-\E\Y_t)^\top(\Y_t-\E\Y_t)] & \preceq \frac{4}{m}\cdot\begin{bmatrix}\Lambdab(\I-\Lambdab) &\zero\\
     \zero & \omega^2\Lambdab(\I-\Lambdab) \end{bmatrix}
\end{align*}
 The last step is to observe that permuting the rows and columns using permutation matrix $\P$ yields the result.
 \end{proof}
In the next result, we construct a recursive framework that will help us in upper bounding the spectral norm of the product of random matrices $\Y_t$. As already exploited in Lemma \ref{l:spec_radius_T} and  Lemma \ref{l:exp_mat}, it is convenient to do the analysis on the transformed matrices (by using permutation $\P$). Consider the following notations:
\begin{align*}
    \Sigmab_t = \prod_{j=0}^{t-1}{\Q\P\Y_{t-j}\P^\top\Q^\top},  \ \ \ \bar\Sigmab_t = \prod_{j=0}^{t-1}{\Q\P\E\Y_t\P^\top\Q^\top}, \ \ \Rb_t = \Q\P(\Y_t-\E\Y_t)\P^\top\Q^\top, \bar\Y = \Q\P\E\Y_t\P^\top\Q^\top,
\end{align*}
where $\P$ is the permutation matrix such that $\P\E\Y_t\P^\top = \diag(\T_1,\T_2,\cdots,\T_n)$.
We prove the following lemma:
\begin{lemma}\label{l:pt_recursion1}
For any $t\geq 1$, there exist $p_0,p_1,...,p_t$ such that $p_0=1$ and,
     \begin{align*}
     \big\| \E[\Sigmab_t^\top\Sigmab_t]\big\| \leq   p_t \leq \max_{1\leq i \leq r}\Big(\big\|(\T_i^\top)^t\T_i^t\big\| + \frac{4\lambda_i(1-\lambda_i)}{m}\cdot\sum_{j=0}^{t-1}{p_j\cdot\big\|(\T_i^\top)^{t-1-j}\T_i^{t-1-j}\big\|}\Big).
    \end{align*}
\end{lemma}
\begin{proof}
We provide a proof by doing induction over $t$. Note that $t=0$ forms the base case, and as $\Sigmab_0 =\I$, the base case hold for $p_0=1$. Now for $k <t$, let $ \big\| \E[\Sigmab_k^\top\Sigmab_k]\big\| \leq p_k$, for $p_k$ satisfying the the inequality in statement of Lemma \ref{l:pt_recursion1}, and let $\M_k=\E[\Sigmab_k^\top\Sigmab_k] - \bar\Sigmab_k^\top\bar\Sigmab_k $. For $\Sigmab_t$ we have,
\begin{align*}
    \E[\Sigmab_t^\top\Sigmab_t] &= \E\big[\Q\P\Y_t^\top\P^\top\Q^\top\E[\Sigmab_{t-1}^\top\Sigmab_{t-1}]\Q\P\Y_t\P^\top\Q^\top\big]\\
    &=  \E\big[\Q\P\E\Y_t^\top\P^\top\Q^\top\E[\Sigmab_{t-1}^\top\Sigmab_{t-1}]\Q\P\E\Y_t\P^\top\Q^\top\big] + \E\big[\Rb_t^\top\E[\Sigmab_{t-1}^\top\Sigmab_{t-1}]\Rb_t\big]\\
    &= \E\big[\bar\Y^\top\E[\Sigmab_{t-1}^\top\Sigmab_{t-1}]\bar\Y\big] + \E\big[\Rb_t^\top\E[\Sigmab_{t-1}^\top\Sigmab_{t-1}]\Rb_t\big]\\
    &= \bar\Sigmab_t^\top\bar\Sigmab_t + \bar\Y^\top\big(\E[\Sigmab_{t-1}^\top\Sigmab_{t-1}]-\bar\Sigmab_{t-1}^\top\bar\Sigmab_{t-1}\big)\bar\Y + \E\big[\Rb_t^\top\E[\Sigmab_{t-1}^\top\Sigmab_{t-1}]\Rb_t\big]
\end{align*}
Rearranging the terms, we get,
\begin{align*}
    \underbrace{\E[\Sigmab_t^\top\Sigmab_t] - \bar\Sigmab_t^\top\bar\Sigmab_t}_{\M_t} = \E\big[\Rb_t^\top\E[\Sigmab_{t-1}^\top\Sigmab_{t-1}^\top]\Rb_t\big] +\bar\Y^\top\big(\E[\Sigmab_{t-1}^\top\Sigmab_{t-1}]-\bar\Sigmab_{t-1}^\top\bar\Sigmab_{t-1}\big)\bar\Y.
\end{align*}
By induction hypothesis $\big\|\E[\Sigmab_{t-1}^\top\Sigmab_{t-1}]\big\| \leq p_{t-1}$, we get,
\begin{align}
    \M_{t} &\preceq p_{t-1}\cdot\E[\Rb_t^\top\Rb_t] + \bar\Y^\top\M_{t-1}\bar\Y, \label{e:m_t_recursion} \\
    \E[\Sigmab_t^\top\Sigmab_t] & \preceq \bar\Sigmab_t^\top\bar\Sigmab_t + p_{t-1}\cdot\E[\Rb_t^\top\Rb_t] + \bar\Y^\top\M_{t-1}\bar\Y,\label{e:pib_t_recursion}
\end{align}
Note that $\E[\Rb_t^\top\Rb_t] = \M_1$. Using Lemma \ref{l:second_moment}, it follows that $\M_1$ is upper bounded by a block-diagonal matrix. Let us define $\D_1 := 4\cdot\diag(\W_1,\W_2,\cdots\W_r)$ where $\W_i's$ are as given in Lemma~\ref{l:second_moment}, and for $k>1$, $\D_k := p_{k-1}\cdot\D_1 + \bar\Y^\top\D_{k-1}\bar\Y$. Note that $\D_1$ is block-diagonal and $\M_1 \preceq \D_1$. Let's proceed inductively and assume that $\D_{k-1}$ is block-diagonal and $\M_{k-1} \preceq \D_{k-1}$. By definition, $\D_k$ is block-diagonal as $\bar\Y$ is block-diagonal and by induction hypothesis $\D_{k-1}$ is block-diagonal. Furthermore, by (\ref{e:m_t_recursion}), $\M_{k} \preceq \D_k$ as $\M_1 \preceq \D_1$ and $\M_{k-1} \preceq \D_{k-1}$ by induction hypothesis. Therefore, by induction, $\forall k, \ \M_{k} \preceq \D_k$, and $\D_k's$ are block-diagonal satisfying $\D_k = p_{k-1}\cdot\D_1 + \bar\Y^\top\D_{k-1}\bar\Y$.
We now write the recursive expression for diagonal blocks in $\D_t$. Let the $i^{th}$ diagonal block of $\bar\Y$ be $\T_i$, and $i^{th}$ diagonal block of $\D_t$ be $\D_{t,i}$. This yields,
\begin{align*}
     \D_{t,i} = p_{t-1}\cdot\D_{1,i}+ \T_i^\top\D_{t-1,i}\T_i.
\end{align*}
Unfolding the above recursion, we get,
\begin{align*}
    \D_{t,i} &=  p_{t-1}\cdot\D_{1,i}  + \T_i^\top(p_{t-2}\cdot\D_{1,i}  + \T_i^\top\D_{t-2,i}\T_i)\T_i\\
    &=p_{t-1}\cdot\D_{1,i} +p_{t-2}\cdot\T_i^\top\D_{1,i}\T_i + (\T_i^\top)^2\D_{t-2,i}\T_i^2.
\end{align*}
Continuing the recursion, we get
\begin{align*}
    \D_{t,i} &= \sum_{j=0}^{t-1}{p_j\cdot(\T_i^\top)^{t-1-j}\D_{1,i}\T_i^{t-1-j}}, \text{or}\\
    \D_{t-1,i} &= \sum_{j=0}^{t-2}{p_j\cdot(\T_i^\top)^{t-2-j}\D_{1,i}\T_i^{t-2-j}}.
\end{align*}
Substituting in (\ref{e:pib_t_recursion}), we bound $\big\|\E\Sigmab_t^\top\Sigmab_t\big\|$ as 
\begin{align*}
    \big\|\E\Sigmab_t^\top\Sigmab_t\big\| &\preceq \max_{1\leq i \leq r}\Big\|(\T_i^\top)^t\T_i^t + p_{t-1}\cdot\D_{1,i} +  \sum_{j=0}^{t-2}{p_j\cdot(\T_i^\top)^{t-1-j}\D_{1,i}\T_i^{t-1-j}}\Big\|\\
    &\preceq \max_{1\leq i \leq r}\Big\|(\T_i^\top)^t\T_i^t + \sum_{j=0}^{t-1}{p_j\cdot(\T_i^\top)^{t-1-j}\D_{1,i}\T_i^{t-1-j}}\Big\|.
\end{align*}
Using triangle inequality and norm submultiplicativity, we get,
\begin{align*}
    \big\|\E\Sigmab_t^\top\Sigmab_t\big\| \leq p_t = \max_{1\leq i \leq r}\Big(\big\|(\T_i^\top)^t\T_i^t\big\| + \|\D_{1,i}\|\cdot\sum_{j=0}^{t-1}{p_j\cdot\big\|(\T_i^\top)^{t-1-j}\T_i^{t-1-j}}\big\|\Big).
\end{align*}
The proof follows by substituting for $\|\D_{1,i}\|$ using Lemma \ref{l:second_moment}.
\end{proof}
Now the task is to upper bound the recursive expression for $p_t$ in Lemma \ref{l:pt_recursion1}. In the next lemma, we explicitly bound the spectral norm of $(\T_i^\top)^k\T_i^k$ for any $i$ and any $ k\geq 1$.
\begin{lemma}\label{l:tct_explicit_bound}
For any $k\geq 0$,
\begin{align*}
    \big\|(\T_i^\top)^k\T_i^k\big\| &\leq 11\cdot|\gamma_{i1}|^{2k}\cdot h_i^2(k).
\end{align*}
where $h_i(k) = \min\big\{\frac{k}{|\gamma_{i1}|}, \frac{2}{|\gamma_{i1}-\gamma_{i2}|} \big\}$ if $\gamma_{i1} \ne \gamma_{i2}$, and $h_i(k) = \frac{k}{|\gamma_{i1}|}$ if $\gamma_{i1} = \gamma_{i2}$.
\end{lemma}
\begin{proof}
We use Lemma \ref{l:schur} to upper bound $\|\T_i^k\|_F^2$. Note that,
    \begin{align*}
        \big\|(\T_i^\top)^k\T_i^k\big\|
        &\leq \|\T_i^k\|_F^2\\
        &=\Big(|\gamma_{i1}|^{2k} + |x_i|^2|g_i(k)|^2 + |\gamma_{i2}|^{2k}\Big),
    \end{align*}
    where $g_i(k) = \sum_{j=0}^{k-1}{\gamma_{i1}^j\gamma_{i2}^{k-1-j}}$. We have, $|g_i(k)| = |\gamma_{i1}|^{k}\cdot\min\big\{\frac{k}{|\gamma_{i1}|}, \frac{|1-(\gamma_{i2}/\gamma_{i1})^k|}{|\gamma_{i1}-\gamma_{i2}|} \big\} \leq |\gamma_{i1}|^k\cdot h_i(k)$. Note that $h_i(k)$ is always larger than $1$. Moreover, as
    $|x_i| \leq 3$, we get,
    \begin{align*}
        \big\|(\T_i^\top)^k\T_i^k\big\| &\leq 11\cdot|\gamma_{i1}|^{2k}\cdot h_i^2(k).
    \end{align*}
\end{proof}
Now suppose $t>0$ and our aim is to upper bound $p_{t+1}$. In the following analysis, we construct an upper bound for $p_{t+1}$, by defining $f_{t+1}$, which itself has a recursive definition, but a simpler one compared to $p_{t+1}$. We will use Lemma \ref{l:pt_recursion1} and Lemma \ref{l:tct_explicit_bound} to build this recursion. For convenience, consider the following notations: 
\begin{align}
\alpha_i(t) := \big\|(\T_i^\top)^t\T_i^t\big\|, \ \  q_i :=\frac{4\lambda_i(1-\lambda_i)}{m},  \ \ \ell_i(t) := \max_{0\leq k\leq t}\Big\{11\cdot|\gamma_{i1}|^k\cdot h_i^2(k)\Big\}, \label{e:ell_i_def}
\end{align}
where $h_i(k)$ is defined in Lemma \ref{l:tct_explicit_bound}.
We have $p_0=1$, and 
\begin{align*}
    p_{t+1} &= \max_{1\leq i \leq r}\Big\{\big\|(\T_i^\top)^{t+1}\T_i^{t+1}\big\| + q_i\cdot\sum_{j=0}^{t}{p_j\cdot\big\|(\T_i^\top)^{t-j}\T_i^{t-j}\big\|}\Big\}\\
    &< \max_{1\leq i \leq r}\Big\{\alpha_i(t+1) +q_i\cdot\sum_{j=0}^{t}{p_j\cdot|\gamma_{i1}|^{t-j}\cdot\ell_i(t-j)}\Big\}\\
    &\leq  \max_{1\leq i \leq r}\Big\{\alpha_i(t+1) +q_i\cdot\ell_i(t)\cdot\sum_{j=0}^{t}{p_j\cdot|\gamma_{i1}|^{t-j}}\Big\}.
\end{align*}
Informed by the above relations, we construct upper bounds for $p_k$ for $0\leq k\leq t+1$, denoted by $f_k$ and defined as follows:
\begin{align*}
    &f_0=1,\\
    &f_{k+1} = \max_{1\leq i \leq r}\Big(\alpha_{i}(k+1) +q_i\cdot\ell_i(t)\cdot\sum_{j=0}^{k}{f_j\cdot|\gamma_{i1}|^{k-j}}\Big) \ \ \text{for $0\leq k \leq t$.}
\end{align*}
Note that $p_{k+1} \leq f_{k+1}$ for all $0 \leq k \leq t $. For any fixed $k$ such that $0\leq k\leq t$,
\begin{align*}
    f_{k+1} &= \max_{1\leq i \leq r}\Big(\alpha_i(k+1) + q_i\cdot\ell_i(t)\cdot\sum_{j=0}^{k}{f_j\cdot|\gamma_{i1}|^{k-j}}\Big)\\
    & = \max_{1\leq i \leq r}\Big(\alpha_i(k+1) + q_i\cdot\ell_i(t)\cdot f_k + |\gamma_{i1}|\cdot \underbrace{q_i\cdot\ell_i(t)\cdot\sum_{j=0}^{k-1}{f_j\cdot|\gamma_{i1}|^{k-1-j}}}_{*}\Big)
\end{align*}
Now note, irrespective of the index $i$, the term marked as $*$ is always upper bounded by $f_k$. Therefore,
\begin{align*}
    f_{k+1} &< \max_{1\leq i\leq r}\Big(\alpha_i(k+1) + \big(|\gamma_{i1}| + q_i \cdot\ell_i(t)\big)\cdot f_k\Big)\\
    &\leq \max_{1\leq i \leq r}\big(|\gamma_{i1}| + q_i \cdot\ell_i(t)\big)\cdot f_k + \max_{1\leq i \leq r}\alpha_i(k+1)\\
    &\leq \rho\cdot f_k + \max_{1\leq i \leq r} \alpha_i(k+1).
\end{align*}
where we let $\rho:=\max_{1\leq i\leq r}\big(|\gamma_{i1}| + q_i \cdot\ell_i(t)\big)$. Now, using the above recursive relation we can upper bound $f_{t+1}$ as,
\begin{align}
    f_{t+1} &\leq \rho\cdot f_t + \max_{1\leq i\leq r}\alpha_i(t+1) \nonumber\\
    &\leq \rho^2\cdot f_{t-1} + \rho\cdot\max_{1 \leq i \leq r}\alpha_i(t) + \max_{1\leq i\leq r}\alpha_i(t+1) \nonumber\\
    &\leq \rho^{t+1} + \sum_{j=0}^{t}{\rho^{j}\cdot\max_{1\leq i \leq r}\alpha_i(t+1-j)}. \label{e:gt_bound}
\end{align}
We need to bound $\rho$. In the following lemma, we provide a bound on the term $|\gamma_{i1}| + q_i\cdot\ell_i(t)$ for any $i$. For simplicity, we provide the proof for $\omega \in \{0,1\}$ (which covers both HBM and NAG). The proof for $\omega \in (0,1)$ is completely analogous and follows by using a more general version of Lemma \ref{l:helper}. 
\begin{lemma}\label{l:rho_bound}
   Let $2\leq \phi \leq \min\big\{\frac{m}{C}, \frac{1}{3\sqrt{\lambda_r}}\big\}$ for $C> 76800$ and $\omega \in [0,1]$. Then for any $1\leq i\leq r$,
   \begin{align*}
        |\gamma_{i1}| +r\cdot \ell_i(t) &\leq 1-\frac{\phi\lambda_r}{4}.
    \end{align*}
\end{lemma}
\begin{proof}
As we are proving the result for a fixed $i$, we drop the subscript $i$ for notational convenience. The proof proceeds in cases:
    \begin{enumerate}
        \item \emph{Case 1: $\gamma_1$ is complex.} We know from Lemma \ref{l:spec_radius_T} that $|\gamma_1|^2 = \Big(1-\frac{1}{\phi}\Big)(1-\omega\lambda)$. We get
        \begin{align}
            |\gamma_1| +q\cdot\ell(t) < \Big(1-\frac{1}{2\phi}\Big) + \frac{4\lambda(1-\lambda)}{m}\cdot\ell(t).  \label{e:comp_0}
        \end{align}
        Recall $\ell(t)= \max_{0\leq k\leq t}\Big\{11\cdot|\gamma_{1}|^k\cdot h^2(k)\Big\}$ and $h(k) =  \min\Big\{\frac{k}{|\gamma_{1}|}, \frac{2}{|\gamma_{1}-\gamma_{2}|} \Big\} $. Let $k^*$ be the index such that $\ell(t) = 11\cdot|\gamma_1|^{k^*}\cdot   h^2(k^*)$. We get,
        \begin{align*}
            |\gamma_1| +r\cdot\ell(t)  < \Big(1-\frac{1}{2\phi}\Big) + \frac{44\lambda(1-\lambda)}{m}\cdot h^2(k^*)\cdot|\gamma_1|^{k^*}.
        \end{align*}
        Now look at $h^2(k^*)$,
        \begin{align*}
            h^2(k^*) = \min\Big\{\frac{(k^*)^2}{|\gamma_1|^2}, \frac{4}{|\gamma_1-\gamma_2|^2}\Big\}.
        \end{align*}
        Due to Lemma \ref{l:spec_radius_T}, we know
        \begin{align*}
            |\gamma_1 - \gamma_2|^2 &= \frac{1}{4}\Big(4\beta (1-\omega\lambda)-\big((1-\lambda) + \beta(1-\omega\lambda)\big)^2\Big)\\
            &=\frac{(1-\lambda)}{4}\cdot(4\beta\delta-(1+\beta\delta)^2(1-\lambda)\big),
        \end{align*}
        where $\delta=\frac{1-\omega\lambda}{1-\lambda}$. Also, as $\beta= 1-\frac{1}{\phi} \geq 0.5$, we have $\frac{1}{|\gamma_1|^2} \leq \frac{2}{1-\omega\lambda} \leq \frac{2}{1-\lambda}$. Therefore,
        \begin{align}
            |\gamma_1| +q\cdot\ell(t) < \Big(1-\frac{1}{2\phi}\Big) + \frac{44\lambda}{m}\cdot\min\Big\{2(k^*)^2, \frac{16}{4\beta\delta-(1+\beta\delta)^2(1-\lambda)}\Big\}\cdot|\gamma_1|^{k^*} \label{e:comp_1}
        \end{align}
        Let $\epsilon= 4\beta\delta - (1+\beta\delta)^2(1-\lambda)$. We get,
        \begin{align}
            1-\lambda &= \frac{4\beta\delta}{(1+\beta\delta)^2} -\frac{\epsilon}{(1+\beta\delta)^2} \nonumber\\
            \lambda &= 1-\frac{4\beta\delta}{(1+\beta\delta)^2} +\frac{\epsilon}{(1+\beta\delta)^2}\nonumber \\
            &= \frac{(1-\beta\delta)^2}{(1+\beta\delta)^2} + \frac{\epsilon}{(1+\beta\delta)^2} \label{e:comp_t1}
        \end{align}
        We consider two subcases:
        \begin{enumerate}
            \item \emph{Subcase-1: $\epsilon > \frac{(1-\beta\delta)^2}{(1+\beta\delta)^2}$}. Using (\ref{e:comp_t1}) we get,
            \begin{align}
                \frac{\lambda}{4\beta\delta-(1+\beta\delta)^2(1-\lambda)} < 1 + \frac{1}{(1+\beta\delta)^2} <2. \label{e:comp_2}
            \end{align}
            \item \emph{Subcase-2:} $\epsilon \leq \frac{(1-\beta\delta)^2}{(1+\beta\delta)^2}$. Using (\ref{e:comp_t1}) we have $\lambda \leq \Big(1+\frac{1}{(1+\beta\delta)^2}\Big)\frac{(1-\beta\delta)^2}{(1+\beta\delta)^2}$. Note that $\delta \in \{1,\frac{1}{1-\lambda}\}$ and $b\geq 0.5$. Using item $1$ from Lemma \ref{l:helper} with $x=\delta$ we get,
            \begin{align*}
                \lambda \leq \Big(1+\frac{1}{(1+\beta)^2}\Big)\frac{(1-\beta)^2}{(1+\beta)^2} < \frac{3}{2(2\phi-1)^2},
            \text{ \ \ \ \ (as $\beta = 1-\frac{1}{\phi}$).}
            \end{align*}
            We get,
            \begin{align*}
                \lambda\cdot(k^*)^2\cdot|\gamma_1|^{k^*} \leq  \frac{3}{2(2\phi-1)^2}\cdot(k^*)^2\cdot|\gamma_1|^{k^*}.
            \end{align*}
            The quantity $(k^*)^2\cdot|\gamma_1|^{k^*}$ is maximized at $k^*=\frac{-2}{\log|\gamma_1|}$, implying
            \begin{align*}
                 \lambda\cdot(k^*)^2\cdot|\gamma_1|^{k^*} \leq \frac{3}{2(2\phi-1)^2}\cdot\frac{4}{(1-|\gamma_1|)^2}\cdot\frac{1}{e^2},
            \end{align*}
            where in the last inequality we used the relation $\frac{-2}{\log|\gamma_1|} < \frac{2}{1-|\gamma_1|}$ and $|\gamma_1|^{k^*} =e^{-2}$, with $e$ being Euler's number. As $|\gamma_1| < 1-\frac{1}{2\phi}$, we have $\frac{1}{(1-|\gamma_1|)^2} < 4\phi^2$. Therefore,
            \begin{align}
                 \lambda\cdot(k^*)^2\cdot|\gamma_1|^{k^*} \leq \frac{3}{2(2\phi-1)^2}\cdot\frac{16\phi^2}{e^2} \leq \frac{3}{4\phi^2}\cdot\frac{16\phi^2}{e^2} <2. \label{e:comp_3}
            \end{align}
        \end{enumerate}
        Depending on the above two subcases, we substitute (\ref{e:comp_2}) and (\ref{e:comp_3}) in (\ref{e:comp_1}) to prove
        \begin{align*}
            |\gamma_1| +q\cdot\ell(t) < \Big(1-\frac{1}{2\phi}\Big) + \frac{44\cdot 32}{m}.
        \end{align*}
        If $\phi \leq \frac{m}{5632}$, we get,
        \begin{align}
             |\gamma_1| +q\cdot\ell(t)< 1-\frac{1}{4\phi}. \label{e:comp_4}
        \end{align}
        \item \emph{Case 2:} $\gamma_1$ is real and $\gamma_1=\gamma_2$. In this case we have,
        \begin{align*}
            \lambda = \frac{(1-\beta\delta)^2}{(1+\beta\delta)^2} \leq \frac{(1-\beta)^2}{(1+\beta)^2} = \frac{1}{(2\phi-1)^2},
        \end{align*}
        where the last inequality follows from item $2$ from Lemma \ref{l:helper}. The proof is exactly same to subcase-2 of the first case (complex $\gamma_1$). Using Lemma \ref{l:spec_radius_T} we have,  $\gamma_1^2 = \Big(1-\frac{1}{\phi}\Big)(1-\omega\lambda)$. Therefore,
        \begin{align*}
            |\gamma_1| + q\cdot\ell(t) < \Big(1-\frac{1}{2\phi}\Big) + \frac{4\lambda(1-\lambda)}{m}\cdot\ell(t).
        \end{align*}
         Recall $\ell(t) = \max_{0\leq k\leq t}\Big\{11\cdot|\gamma_{1}|^k\cdot h^2(k)\Big\}$ and $h(k) = \frac{k}{|\gamma_1|} $. Let $k^*$ be the index such that $\ell(t) = 11\cdot|\gamma_1|^{k^*}\cdot h^2(k^*)$. We get,
         \begin{align*}
             |\gamma_1| +q\cdot\ell(t) &< \Big(1-\frac{1}{2\phi}\Big)+\frac{44\lambda(1-\lambda)}{m}\cdot\frac{(k^*)^2}{|\gamma_1|^2}\cdot|\gamma_1|^{k^*}\\
             &<\Big(1-\frac{1}{2\phi}\Big) +\frac{176\lambda}{m}\cdot(k^*)^2\cdot|\gamma_1|^{k^*} \ \  \ \ \ \ \ \ \ \text{\big(as $\gamma_1^2 > (1-\lambda)\sqrt{1-\frac{1}{\phi}}$\big)}\\
             &<\Big(1-\frac{1}{2\phi}\Big) + \frac{176}{m}\cdot\frac{1}{(2\phi-1)^2}\cdot\frac{16\phi^2}{e^2}\\
             & < \Big(1-\frac{1}{2\phi}\Big) + \frac{200}{m}.
         \end{align*}
         Therefore, for $\phi < \frac{m}{400}$, we get
         \begin{align}
              |\gamma_1| +q\cdot\ell(t) < 1-\frac{1}{4\phi}. \label{e:same_1}
         \end{align}
        
         \item \emph{Case 3:} $\gamma_1$ is real with $\gamma_1 > \gamma_2$. Using Lemma \ref{l:spec_radius_T}, we have $\gamma_1 < \Big(1-\frac{\phi\lambda}{2}\Big)(1-\omega\lambda)$. Therefore,
         \begin{align*}
            \gamma_1 +q\cdot\ell(t) < \Big(1-\frac{\phi\lambda}{2}\Big) + \frac{4\lambda(1-\lambda)}{m}\cdot\ell(t).
         \end{align*}
            Here again, $\ell(t) = \max_{0\leq k\leq t}\Big\{11\cdot|\gamma_{1}|^k\cdot h^2(k)\Big\}$ and $h(k) =  \min\Big\{\frac{k}{\gamma_1}, \frac{2}{\gamma_1-\gamma_2} \Big\} $.   Let $k^*$ be the index such that $\ell(t) = 11\cdot |\gamma_1|^{k^*}\cdot h^2(k^*)$. We get,
            \begin{align}
                \gamma_1 +r\cdot\ell(t) < \Big(1-\frac{\phi\lambda}{2}\Big)+ \frac{44\lambda(1-\lambda)}{m}\cdot h^2(k^*)\cdot\gamma_1^{k^*},\label{e:real_0}
            \end{align}
             where,
        \begin{align*}
            h^2(k^*) = \min\Big\{\frac{(k^*)^2}{\gamma_1^2}, \frac{4}{(\gamma_1-\gamma_2)^2}\Big\}.
        \end{align*}
         We have \begin{align*}
             (\gamma_1-\gamma_2)^2 &= \Big(\big((1-\lambda) + \beta(1-\omega\lambda)\big)^2-4\beta (1-\omega\lambda)\Big)\\
             &= (1-\lambda)\Big((1+\beta\delta)^2(1-\lambda)-4\beta\delta\Big),
         \end{align*}
         where $\delta = \frac{1-\omega\lambda}{1-\lambda}$. Let $\epsilon = (1+\beta\delta)^2(1-\lambda)-4\beta\delta.$ Note that this yields
         \begin{align*}
             \lambda = \frac{(1-\beta\delta)^2}{(1+\beta\delta)^2} - \frac{\epsilon}{(1+\beta\delta)^2}.
         \end{align*}
             Moreover, let $a(\lambda) =\lambda- \frac{(1-\beta\delta)^2}{(1+\beta\delta)^2}$ and let $\lambda'$ denote the zero of $a(\lambda)$ in $(0,1)$. It turns out that $a(\lambda)$ is an increasing function for all $\lambda$ satisfying $a(\lambda)<0$. As $\lambda$ in our case satisfies $a(\lambda)<0$, we get
             \begin{align}
                 \lambda < \lambda' < \frac{(1-\beta)^2}{(1+\beta)^2} = \frac{1}{(2\phi-1)^2} \leq \frac{1}{9}.\label{e:lam_upper}
             \end{align}
             where we used item $2$ from Lemma \ref{l:helper} again for the second last inequality. The last inequality used $\phi \geq 2$. Consider the following two subcases:
         \begin{enumerate}
             \item \emph{Subcase 1: $0< \epsilon \leq \frac{(1-\beta\delta)^2}{(1+\beta\delta)^2}$.} We get,
             \begin{align*}
             \lambda \geq \Big(1-\frac{1}{(1+\beta\delta)^2}\Big)\frac{(1-\beta\delta)^2}{(1+\beta\delta)^2}.
             \end{align*}
             Using item $3$ from Lemma \ref{l:helper} we get
             \begin{align}
                 \lambda > \frac{(1-\beta)^2}{2(1+\beta)^2} = \frac{1}{2(2\phi-1)^2} > \frac{1}{8\phi^2}.\label{e:lam_lower}
             \end{align}
         
               We know that (see proof of Lemma \ref{l:spec_radius_T}),
        \begin{align*}
            \gamma_1 > \frac{1}{2}\Big((1-\lambda) + \beta(1-\omega\lambda)\Big) > \frac{1}{2}(1+\beta)(1-\lambda) > \frac{3}{4}(1-\lambda) > \frac{2}{3} \ \ \ \ \text{(by (\ref{e:lam_upper}))}
        \end{align*}
             Therefore,
             \begin{align*}
                 \gamma_1 + q\cdot\ell(t) &< \Big(1-\frac{\phi\lambda}{2}\Big) + \frac{44\lambda(1-\lambda)}{m}\cdot\frac{(k^*)^2}{\gamma_1^2}\cdot\gamma_i^{k^*}\\
                 & < \Big(1-\frac{\phi\lambda}{2}\Big) + \frac{100\lambda}{m}\cdot(k^*)^2\cdot\gamma_1^{k^*} \ \ \ \ \text{(as $\gamma_1 > \frac{2}{3}$)}.
             \end{align*}
 As $(k^*)^2\cdot\gamma_1^{k^*}$ is maximized at $k^* = \frac{-2}{\log\gamma_1}$ and $\gamma_1 < 1-\frac{\phi\lambda}{2}$, we get
            \begin{align*}
                \gamma_1 +q\cdot\ell(t) < \Big(1-\frac{\phi\lambda}{2}\Big) + \frac{100\lambda}{m}\cdot\frac{3}{\phi^2\lambda^2}.
            \end{align*}
           Finally using (\ref{e:lam_lower}) we get,
           \begin{align}
               \gamma_1 +q\cdot\ell(t) &< \Big(1-\frac{1}{16\phi}\Big) + \frac{2400}{m}\nonumber \\
               & < 1-\frac{1}{32\phi} \label{e:real_1},
           \end{align}
           where last inequality follows for $\phi \leq \frac{m}{76800}$.
             \item \emph{Subcase 2:} $\epsilon > \frac{(1-\beta\delta)^2}{(1+\beta\delta)^2}$.
             In this case we can simply lower bound $(\gamma_1-\gamma_2)^2$. Note that,
             \begin{align*}
                 \lambda < \Big(1-\frac{1}{(1+\beta\delta)^2}\Big)\cdot\frac{(1-\beta\delta)^2}{(1+\beta\delta)^2}.
             \end{align*}
             Using item 4 from Lemma \ref{l:helper} with $x =\delta$ we get $\frac{(1-\beta\delta)^2}{(1+\beta\delta)^2} \geq \frac{1}{2}\cdot\frac{(1-\beta)^2}{(1+\beta)^2}$. Therefore,
             \begin{align*}
                 (\gamma_1-\gamma_2)^2 &= (1-\lambda) \cdot \epsilon\\
                 & >  (1-\lambda) \cdot\frac{(1-\beta\delta)^2}{(1+\beta\delta)^2}\\
                 & \geq (1-\lambda) \cdot\frac{(1-\beta)^2}{(1+\beta)^2} = (1-\lambda)\cdot\frac{1}{(2\phi-1)^2}.
             \end{align*}
             Substituting the above in (\ref{e:real_0}) we get,
             \begin{align*}
                 \gamma_1 + q\cdot\ell(t) &< 1-\frac{\phi\lambda}{2} + \frac{176\lambda}{m}\cdot(2\phi-1)^2\\
            &<1-\frac{\phi\lambda}{2} + \frac{704\lambda\phi^2}{m}
             \end{align*}
             For $\phi <\frac{m}{2816}$ we get,
             \begin{align}
                  \gamma_1 + q\cdot\ell(t) < 1-\frac{\phi\lambda}{4} \label{e:real_2}.
             \end{align}
             \end{enumerate}
             
         \end{enumerate}
     
    Combining the inequalities (\ref{e:comp_3}), (\ref{e:comp_4}), (\ref{e:real_1}) and (\ref{e:real_2}) we get for $\phi \leq \min\{\frac{m}{C}, \frac{1}{3\sqrt{\lambda_r}}\}$ and $C> 76800$,
    \begin{align*}
        |\gamma_1| +r\cdot \ell(t) \leq 1- \frac{\phi\lambda}{4}.
    \end{align*}
    This finishes the proof.
\end{proof}
\paragraph{Finishing the proof of Theorem \ref{t:L2_iterate}.} We are now ready to complete the proof of Theorem \ref{t:L2_iterate}. Recall the bound on $p_{t+1}$ in (\ref{e:gt_bound}). We had,
\begin{align*}
    p_{t+1} \leq f_{t+1}  < \rho^{t+1} + \sum_{j=0}^{t}{\rho^{j}\cdot\max_{i}\alpha_i(t+1-j)}
\end{align*}
Due to Lemma \ref{l:rho_bound} we know,
\begin{align*}
    \rho < 1-\frac{\phi\lambda_r}{4}.
\end{align*}
 Furthermore,  $\alpha_i(t+1-j) = \big\|(\T_i^\top)^{t+1-j}\T_i^{t+1-j}\big\|$, and we know from Lemma \ref{l:schur} that
\begin{align*}
     \big\|(\T_i^\top)^{t+1-j}\T_i^{t+1-j}\big\| &\leq \big(3(t+1-j)^2 + 2\big)\cdot|\gamma_{i1}|^{t-j} \\
     &< \big(3(t+1)^2 +2\big)\cdot|\gamma_{i1}|^{t-j}\\
     &<\big(3(t+1)^2 +2\big)\cdot\rho^{t-j}.
\end{align*}
Therefore,
\begin{align*}
    p_{t+1} &< \rho^{t+1} + (t+1)\cdot \big(3(t+1)^2 +2)\cdot\rho^{t} \\
    &< 2(t+1)\cdot\big(3(t+1)^2 +2\big)\cdot\rho^t\\
    &<10(t+1)^3\cdot\rho^t.
\end{align*}
Note that $p_t$ upper bounds $\Big\|\E\Big[\Big(\prod_{j=1}^{t}{\Q\P\Y_j^\top\P^\top\Q^\top}\Big)\Big(\prod_{j=0}^{t-1}{\Q\P\Y_{t-j}\P^\top\Q^\top}\Big)\Big]\Big\|$ and this finishes the proof of Theorem \ref{t:L2_iterate}.

The following helper lemma is used in the proof of Lemma \ref{l:rho_bound} for $\omega=0$. A more general variant can be used to prove Lemma \ref{l:rho_bound} for any $\omega \in (0,1)$. The proof of Lemma \ref{l:helper} follows by a routine calculus exercise.

\begin{lemma}\label{l:helper}
Let \(x \ge 1\), \( \tfrac12 \le \beta < 1\), and define
\begin{align*}
f(x) = \left(1+\frac{1}{(1+\beta x)^2}\right)\frac{(1-\beta x)^2}{(1+\beta x)^2}, \ \ g(x) = \left(1-\frac{1}{(1+\beta x)^2}\right)\frac{(1-\beta x)^2}{(1+\beta x)^2} 
\end{align*}
\begin{enumerate}
    \item Over the set of all $x \ge 1$ satisfying
$1-\frac{1}{x} \leq f(x)$,
the function $f(x)$ is maximized at $x=1$. In particular, $f(x) < \frac{3}{2}\cdot\frac{(1-\beta)^2}{(1+\beta)^2}$.
\item Over the set of all $x\geq 1$ satisfying $1-\frac{1}{x} = \frac{(1-\beta x)^2}{(1+\beta x)^2}$, we have $\frac{(1-\beta x)^2}{(1+\beta x)^2} \leq \frac{(1-\beta)^2}{(1+\beta)^2}$.
\item Over the set of all $x \geq 1$ satisfying $1-\frac{1}{x}< \frac{(1-\beta x)^2}{(1+\beta x)^2}$, $\inf_x g(x) > \frac{4}{5}\cdot\frac{(1-\beta)^2}{(1+\beta)^2}$.
\item Over all the set of all $x \geq 1$ satisfying $1-\frac{1}{x} < g(x)$, $\inf_{x}\frac{(1-\beta x)^2}{(1+ \beta x)^2} > \frac{1}{2}\cdot\frac{(1-\beta)^2}{(1+\beta)^2}$.
\end{enumerate}

\end{lemma}

\section{Additional experiments}\label{s:add_experiments}
In this section, we provide additional details about our experimental setup, and include numerical results on two datasets to complement the results in Section \ref{s:experiment}. Finally, we also explore the dependence of classical momentum on the $\omega$ parameter, which interpolates between HBM and NAG.

\subsection{Experimental setup}
We transform the data matrices through \texttt{StandardScaler} from \texttt{sklearn.preprocessing}. In particular, after sampling $n=2048$ rows from the dataset, we perform the following transformation,
\begin{verbatim}
scaler = StandardScaler()
X_scaled = scaler.fit_transform(X)
K = rbf_kernel(X_scaled, gamma=0.1)
\end{verbatim}

We consider the following values for $\lambda$ and $\epsilon$ tolerance in our experiment:
\begin{enumerate}
    \item \texttt{Covtype}, $\lambda = 0.5$, $\epsilon = 10^{-7}$.
    \item \texttt{California Housing}, $\lambda = 0.16$, $\epsilon = 10^{-7}$.
    \item \texttt{Abalone}, $\lambda =20$, $\epsilon = 10^{-10}$.
    \item \texttt{Phoneme}, $\lambda =10^{-10}$, $\epsilon = 10^{-6}$.
\end{enumerate}
All experiments were conducted on a MacBook Air (M3) with 16 GB of memory.

\subsection{Plots on \texttt{Abalone} and \texttt{Phoneme} datasets}
\begin{figure}[t]
  \centering

  \begin{subfigure}[t]{0.49\linewidth}
    \centering
    \includegraphics[width=\linewidth]{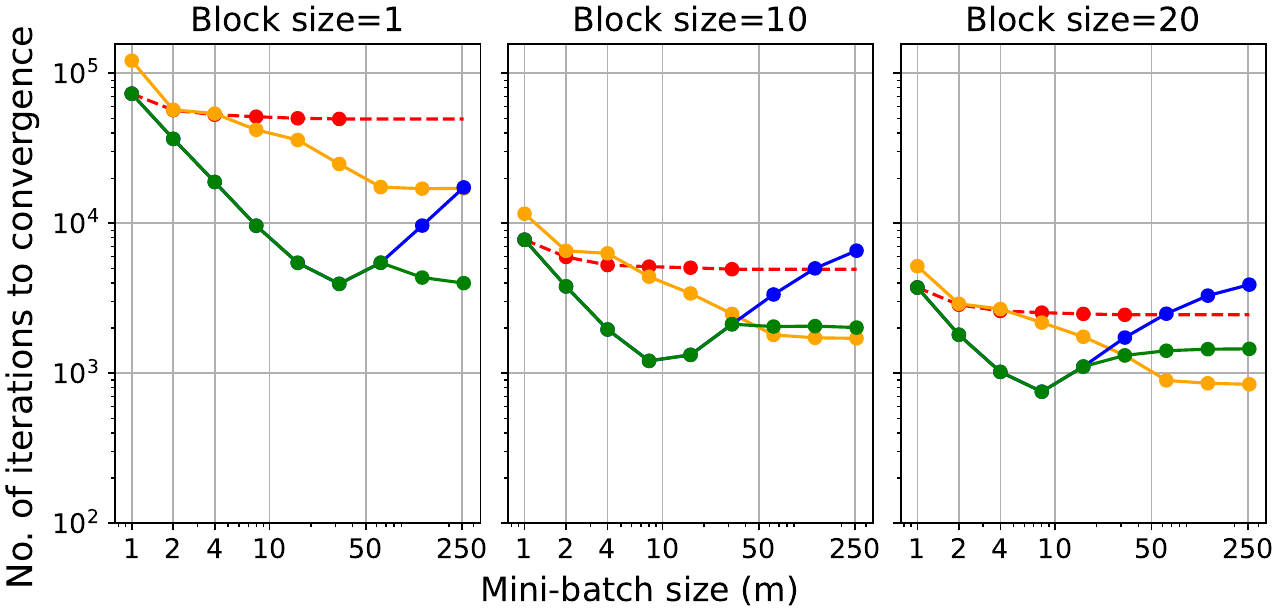}
    \caption{Abalone: iterations vs mini-batch size}
  \end{subfigure}
  \hfill
  \begin{subfigure}[t]{0.49\linewidth}
    \centering
    \includegraphics[width=\linewidth]{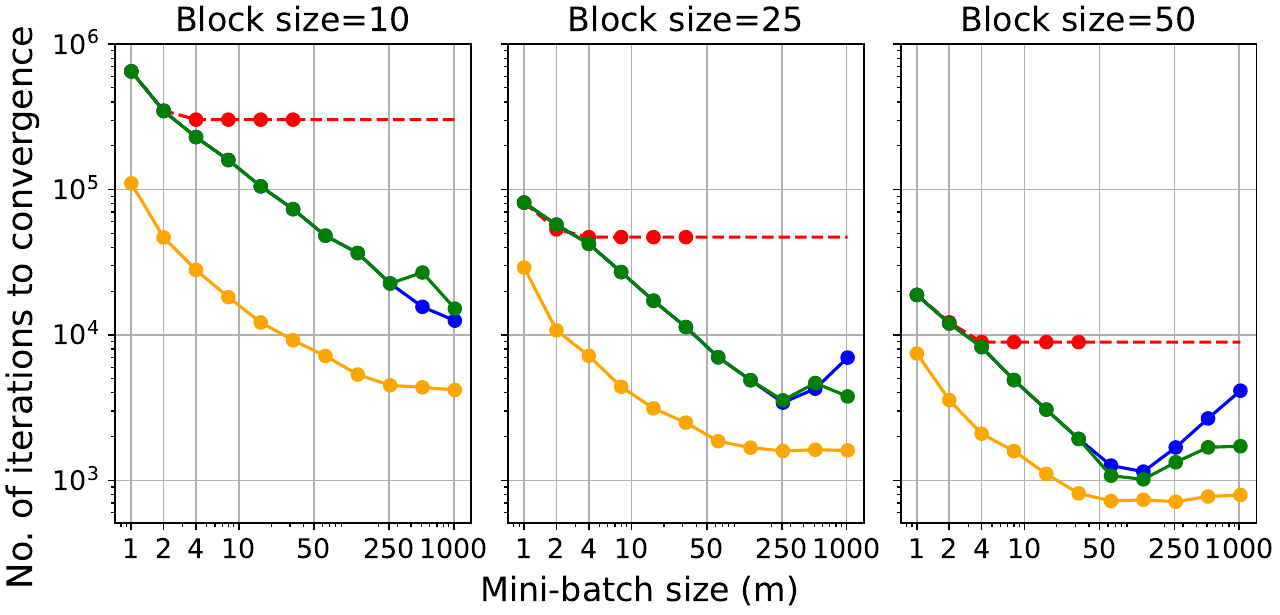}
    \caption{Phoneme: iterations vs mini-batch size}
  \end{subfigure}

  \vspace{0.3em}

  \begin{subfigure}[t]{0.49\linewidth}
    \centering
    \includegraphics[width=\linewidth]{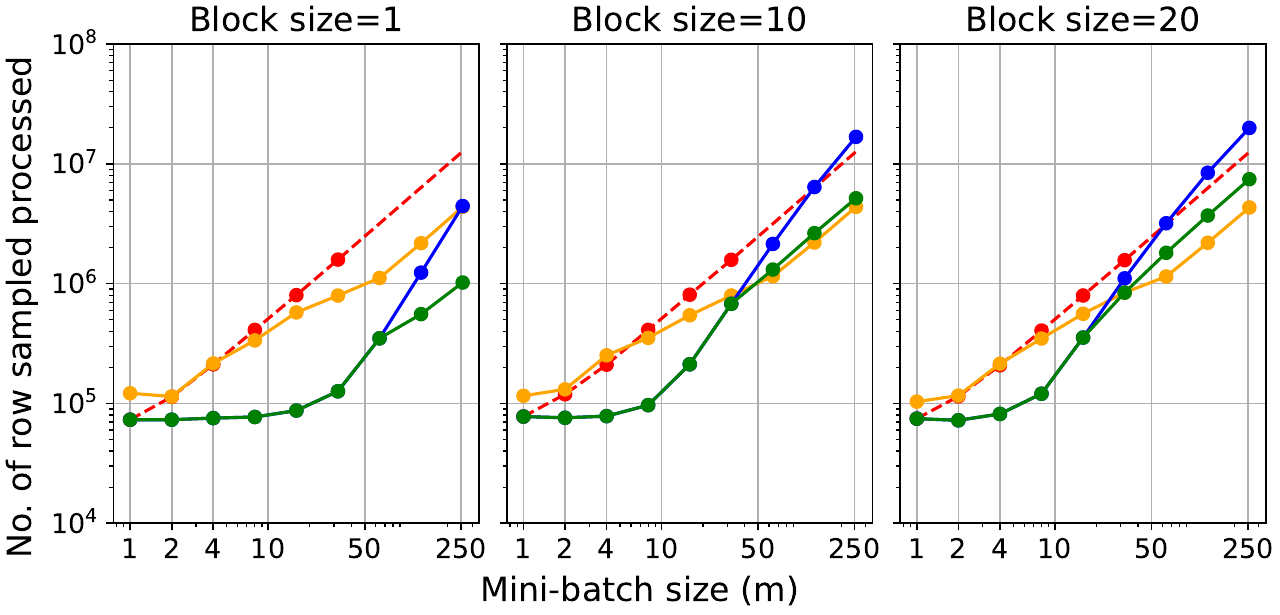}
    \caption{Abalone: row samples vs mini-batch size}
  \end{subfigure}
  \hfill
  \begin{subfigure}[t]{0.49\linewidth}
    \centering
    \includegraphics[width=\linewidth]{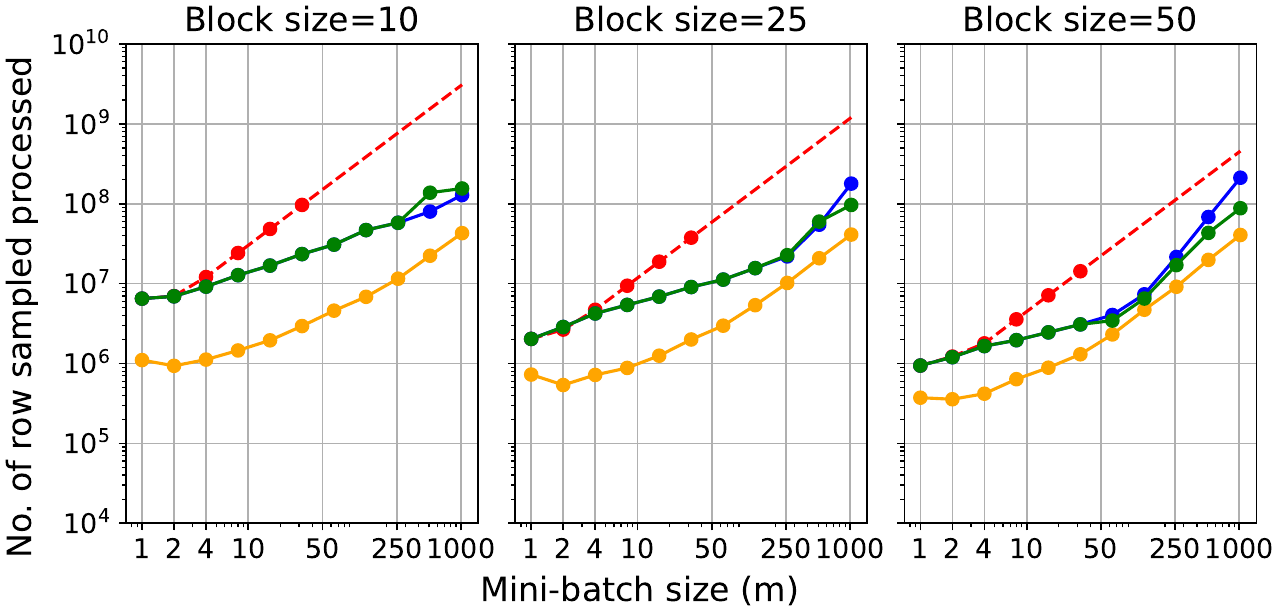}
    \caption{Phoneme: row samples vs mini-batch size}
  \end{subfigure}

  \vspace{0.5em}

  \begin{minipage}{0.9\linewidth}
    \centering
    \small
    \textcolor{myred}{\rule{1em}{1em}} CD \quad
    \textcolor{myorange}{\rule{1em}{1em}} CDpp \quad
    \textcolor{myblue}{\rule{1em}{1em}} CD  $+$ NAG $(\beta=1-1/m)$ \quad
    \textcolor{mygreen}{\rule{1em}{1em}}  CD  $+$ NAG $(\beta=\text{adaptive})$
  \end{minipage}

  \caption{
  Convergence of block coordinate descent with NAG momentum and baselines.
  Top row: iterations to reach error $10^{-10}$ for Abalone and $10^{-6}$ for Phoneme. Bottom row: total work, measured by the number of sampled rows.
  }
  \label{fig:abalone_Phoneme}
\end{figure}

In Fig~\ref{fig:abalone_Phoneme}, we compare, \textbf{CD}, \textbf{CDpp}, \textbf{CD+NAG} ($\beta=1-1/m$), and \textbf{CD+NAG} ($\beta = \text{adaptive}$) on two additional datasets. The results on \texttt{Abalone} dataset follow the same pattern as observed in Fig~\ref{fig:california_covtype} for \texttt{California Housing} and \texttt{Covtype} datasets. However, the results on \texttt{Phoneme} dataset look different, as \textbf{CDpp} performs better than \textbf{CD+NAG}. We provide insights into this phenomenon in the following paragraph.

\paragraph{Comparison between CDpp and CD+NAG.} Under certain incoherence assumptions on the data matrix $\K$, the condition number of $\bar\Pib$ can be bounded as $O\big(\frac{n\bar\kappa_k}{k}\big)$, where $k$ is the block-size and $\bar\kappa_k$ denotes a certain stochastic condition number quantity (see \cite{derezinski2025randomized} for details). In this case,  the iteration complexity of \textbf{CDpp} with mini-batch size $m$ can be shown to scale as $\tilde O\big(\sqrt{\frac{\bar\kappa_k}{m}}\cdot\frac{n}{k}+\sqrt{\bar\kappa_k\frac nk}\big)$, whereas the iteration complexity of \textbf{CD+NAG} scales as $\tilde O\big(\frac{\bar\kappa_k}{m}\cdot\frac{n}{k} +\sqrt{\bar\kappa_k\frac nk}\big)$. Therefore, when $m < \bar\kappa_k$, \textbf{CDpp} can converge faster than \textbf{CD+NAG}, but when $m\ge \bar\kappa_k$, the perfect parallelization phenomenon takes over and \textbf{CD+NAG} becomes faster. While this discussion requires some additional assumptions on the data, it provides some intuition for why one type of acceleration may be preferable over the other, depending on the conditioning properties of the dataset.
\subsection{Adaptive momentum}
Next, we describe the implementation of \textbf{CD+NAG} $(\beta=\text{adaptive})$, where we select the momentum parameter $\beta$ adaptively during the course of the optimization. Consider solving the positive definite linear system $\K\w=\y$ using block coordinate descent with NAG momentum, and with mini-batch size $m$ (here, for simplicity we absorb the regularizer $\lambda$ into $\K$). Our aim is to adaptively find $\phi^* \in [2,m]$ such that $\beta^* = 1-\frac{1}{\phi}$ is a near-optimal momentum parameter for mini-batch size $m$. 

We provide the full heuristic logic for adaptive momentum tuning in Algorithm \ref{alg:adaptive_nag}. The matrix $\tilde\K$ and vector $\tilde\b$ are simply a few randomly sampled rows from $\K$ and $\b$, used to approximate the true residuals. In all our experiments, we sampled $100$ rows randomly from the data matrices for approximating the true residual.
 \begin{algorithm}[H]
\caption{Adaptive momentum tuning for \textbf{CD+NAG}}
\label{alg:adaptive_nag}
\begin{algorithmic}[1]
\Require $\K$, $\tilde\K$, $\b$, $\tilde\b$, $m$, maximum iterations $T$
\Require Warm-up length $T_{\rm warm}$, check interval $T_{\rm check}$
\State Initialize $\w_{-1}=0$, $\x_{-1}=0$
\State Set initial momentum $\beta \gets 1-\frac{1}{m}$

\Comment{Warm-up phase}
\For{$t=0,\ldots,T_{\rm warm}-1$}
    \State $\x_{t+1} \gets \textbf{mini-batch CD step}(\w_t,m)$
    \State $\w_{t+1} \gets \x_{t+1} + \beta(\x_{t+1}-\x_{t})$
\EndFor

\Comment{Adaptive phase}
\State $ \beta_{+}=1-\frac{1}{m}$
\State $ \beta_{-}=1-\frac{1}{2}$
\State $ (\w_{+,t},\x_{+,t}) \gets (\w_t,\x_t)$
\State $ (\w_{-,t},\x_{-,t}) \gets (\w_t,\x_t)$

\While{$t<T$ and $\frac{1-\beta_{-}}{1-\beta_{+}} >2$}
    \For{$i \in \{+,-\}$}
        \State $\x_{i,t+1} \gets \textbf{mini-batch CD step}(\w_{i,t},m)$
        \State $\w_{i,t+1} \gets \x_{i,t+1} + \beta_{i}(\x_{i,t+1}-\x_{i,t})$

     \If{$t \bmod T_{\text{check}} = 0$}
        \State $r_i \gets \|\tilde\K\x_{i,t+1}-\tilde\b\|/\|\tilde\b\|$

    \State 
        $R \gets r_{+}/r_{-}$
    \If {$R <1$}
     $p \gets p+1$ \Else\ $p \gets 0$
    \EndIf
    \If{$R \geq 3$}
        \Comment{Large momentum is too aggressive}
        
        \State $ \phi_{+} \gets \frac{1}{2}\frac{1}{1-\beta_{+}}, \qquad \beta_{+}\gets 1-\frac{1}{\phi_{+}}$
        
        \State $\w_{+,t+1} \gets \w_{-,t+1}, \qquad \x_{+,t+1} \gets \x_{-,t+1}$
        \State $p \gets 0$
    \ElsIf{$R \leq 1/2$ or $p>20$}
      \Comment{Small momentum is too conservative}
       \State $ \phi_{-} \gets \frac{2}{1-\beta_{-}}, \qquad \beta_{-}\gets 1-\frac{1}{\phi_{-}}$
        
        \State $\w_{-,t+1} \gets \w_{+,t+1}, \qquad \x_{-,t+1} \gets \x_{+,t+1}$
        \State $p \gets 0$
    \EndIf
        
     \EndIf
     \EndFor
    
    \State $t \gets t+1$
\EndWhile
\State $\beta \gets \beta_{+}$
\State $(\w_t,\x_t) \gets (\w_{+,t}, \x_{+,t})$

\While{$t<T$}
\Comment{\textbf{CD+NAG} with optimal $\beta$}
        \State $\x_{t+1} \gets \textbf{mini-batch CD step}(\w_{t},m)$
        \State $\w_{t+1} \gets \x_{t+1} + \beta(\x_{t+1}-\x_{t})$
        \State $t \gets t+1$
\EndWhile
\State \Return $\w_t$
\end{algorithmic}
\end{algorithm}
\section{Proofs omitted from Section \ref{s:framework}}\label{s:prop_proofs}
\paragraph{Proof of Proposition \ref{p:simple-rate}}As $\Deltab_{t+1}=\mathcal A_t(\Deltab_t;\alpha) = (\I-\alpha\Pib_t)\Deltab_t$, we have
\begin{align*}
\E\|\Deltab_{t+1}\|^2 &= \Deltab_t^\top\E[(\I-\alpha\Pib_t)^2]\Deltab_t\\
&\leq \Deltab_t^\top\E[(\I-\alpha\Pib_t)]\Deltab_t \qquad\text{(since $\zero \preceq \Pib_t \preceq \I$ and $0 <\alpha \leq 1$)}\\
&=\Deltab_t^\top(\I-\alpha\bar\Pib)\Deltab_t.
\end{align*}
As $\Deltab_0\in\mathrm{range}(\bar\Pib)$, all $\Deltab_t \in \mathrm{range}(\bar\Pib)$ with probability $1$. Therefore, we get
\begin{align*}
    \E\|\Deltab_{t+1}\|^2 \leq \Big(1-\frac{1}{\kappa}\Big)\|\Deltab_0\|^2,
\end{align*}
where $\kappa = \alpha/\lambda^+_{\min}(\bar\Pib)$.

\paragraph{Proof of Proposition \ref{p:mini-batching}} 
It is easy to see that
 \begin{align*}
     \Pib^{[m]} = \frac{1}{m}\sum_{i=1}^{m}{\Pib^{(i)}},
 \end{align*}
 where $\Pib^{(i)}$ is the stochastic rate matrix associated with $\mathcal{A}^{(i)}$. Therefore,
 \begin{align*}
     \E\Big[(\Pib^{[m]} - \bar\Pib)^2\Big] &= (\Pib^{[m]})^2  + \bar\Pib^2 -\E[\Pib^{[m]}]\cdot\bar\Pib - \bar\Pib\cdot\E[\Pib^{[m]}]\\
     & = \big(\Pib^{[m]}\big)^2 - \bar\Pib^2\\
     &= \frac{1}{m^2}\cdot\sum_{i=1}^{m}{\E\big[\big(\Pib^{(i)}\big)^2\big]} + \frac{1}{m^2}\cdot\sum_{i,j=1,  i\ne j}^{m}{\E\big[\Pib^{(i)}\Pib^{(j)}\big]} -\bar\Pib^2\\
     &= \frac{1}{m}\cdot\E\big[\big(\Pib^{(1)}\big)^2\big] - \frac{1}{m}\cdot\bar\Pib^2\\
     &\preceq \frac{1}{m}\cdot\bar\Pib(\I-\bar\Pib).
 \end{align*}

\end{document}